\documentclass[12pt]{article}
   \usepackage{setspace}
   \usepackage[left=1in,right=1in,top=1in,bottom=1in]{geometry}

\usepackage{color}
\usepackage[colorlinks,citecolor=blue,urlcolor=blue, linkcolor=blue, bookmarks=false,hypertexnames=true]{hyperref}
\usepackage{url}
\usepackage{float}
\usepackage{graphicx}
\usepackage{bbm}
\usepackage{doi} 

\usepackage[]{graphicx,float,latexsym}
\usepackage{amsfonts,amstext,amsmath,amssymb,amsthm}
\usepackage{booktabs,array,multirow}
\usepackage{caption}
\usepackage{multirow}

\usepackage[ruled,vlined]{algorithm2e}
\usepackage{subcaption}

\newcommand{\Tsh}{T_{\scriptscriptstyle \text{Sh}}} 
\newcommand{\Tc}{T_{\scriptscriptstyle \text{C}}}
\newcommand{\Tg}{T_{\scriptscriptstyle \text{G}}}
\newcommand{\Ts}{T_{\scriptscriptstyle \text{S}}}
\newcommand{\IKL}{I_{\scriptscriptstyle \text{KL}}}

\usepackage{url} 

\theoremstyle{plain} 
\newtheorem{theorem}{Theorem}[section]
\newtheorem{lemma}{Lemma}[section]

\theoremstyle{definition} 

\newtheorem{remark}{Remark}[section]

\numberwithin{equation}{section} 

\usepackage[authoryear]{natbib}

\title{\bf Sequential Change-Point Detection for Mutually Exciting Point Processes}

\author{
Haoyun Wang$^\mathrm{a}$, Liyan Xie$^\mathrm{a}$, Yao Xie$^\mathrm{a}$, Alex Cuozzo$^\mathrm{b}$, Simon Mak$^\mathrm{b}$\\
        \small $^{a}$School of Industrial and Systems Engineering, Georgia Institute of Technology, Atlanta, Georgia, USA, \\ 
        \small $^{a}$Department of Statistical Science, Duke University, Durham, North Carolina, USA \\
}

\date{} 

\begin{document}

\maketitle

\bigskip
\begin{abstract}
We present a new CUSUM procedure for sequential change-point detection in self- and mutually-exciting point processes (specifically, Hawkes networks) using discrete events data. Hawkes networks have become a popular model in statistics and machine learning, primarily due to their capability in modeling irregularly observed data where the timing between events carries a lot of information. The problem of detecting abrupt changes in Hawkes networks arises from various applications, including neuroengineering, sensor networks, and social network monitoring. Despite this, there has not been an efficient online algorithm for detecting such changes from sequential data. To this end, we propose an online recursive implementation of the CUSUM statistic for Hawkes processes, which is computationally and memory-efficient and can be decentralized for distributed computing. We first prove theoretical properties of this new CUSUM procedure, then show the improved performance of this approach over existing methods, including the Shewhart procedure based on count data, the generalized likelihood ratio statistic, and the standard score statistic. This is demonstrated via simulation studies and an application to population code change-detection in neuroengineering.
\end{abstract}

\noindent%
\vfill

\newpage

\section{Introduction}

Point processes are widely used for modeling discrete events data, which consists of a series of event times and additional associated information. Recently, a class of mutually-exciting non-homogeneous point processes called Hawkes processes \citep{hawkes1971spectra} has gained much popularity in the statistics and machine learning literature. The intensity function of the Hawkes process consists of a deterministic part and a stochastic part, which captures the triggering or inhibiting effects of past events on future events. For example, each earthquake is usually followed by a sequence of aftershock activities and the occurrence rate of aftershocks can be represented in the stochastic part of the intensity function \citep{ogata1988statistical}.  
Hawkes processes provide a flexible model for capturing spatio-temporal correlations, and have been successfully applied in a wide range of domains including seismology \citep{ogata1988statistical,ogata1998space}, criminology \citep{mohler2011self}, epidemiology \citep{rizoiu2018sir}, social networks \citep{yang2013mixture}, finance \citep{hawkes2018hawkes}, and neural activity \citep{reynaud2013inference}. 

Detection of abrupt changes in the Hawkes process is a fundamental problem, which aims to detect the change as quickly as possible subject to false alarm constraints. For instance, in sensor network monitoring, we would like to detect any change as soon as possible using a stream of event data; such changes may represent a shift in system status or event anomalies. There are, however, key challenges for detecting changes in Hawkes processes; this includes the complex spatial and temporal dependence of the event data and long-term dependencies. To address such challenges, we need to develop computationally efficient online detection algorithms with performance guarantees. 

A motivating application for our work is the change-point detection of biological neural networks. This is a fundamental topic in neuroengineering \citep{eliasmith2003neural}, an emerging area at the intersection of physical and biological sciences. The goal is to detect neural states and state changes from experimental spike train data, which records the sequence of times when a neuron fires an action potential. Hawkes processes provide an appealing model for such data: its mutually exciting property naturally mimics neuron-to-neuron influence's electrochemical dynamics. The model's probabilistic nature can also capture noisy influences on the network, resulting from unobserved neurons or external stimuli. There has been much work on applying Hawkes processes for neuroscience problems, e.g., for inferring functional connectivity \citep{lambert2018reconstructing} and uncertainty quantification \citep{wang2020uncertainty}. Change-points over a biological network often arise from sparse population code changes \citep{tang2018large}. Figure~\ref{fig:codingillus} illustrates an example of this change. Here, each dot represents a neuron in the visual cortex. Colored dots show neurons that respond to seeing a cat or a dog, and shared dots represent common features between both animals (e.g., mammal, pet). The sequential detection of this change-point sheds light on the relationship between stimulus and response timing, providing a better understanding of each neuron's role in the population code, which can then be used for rehabilitation of neuronal networks.

Other applications include detecting the existence of hot topics over social networks which are of interest for the social media \citep{li2017detecting}. Change-point detection of the distribution of crime cases can help the police department to take quick reaction and reallocate patrols \citep{mohler2011self}. Detecting changes in the spread of COVID cases will enable people to recognize a threathening new-born variant \citep{chiang2021hawkes}.

\begin{figure}[htbp!]
\centering
\includegraphics[width = 0.6\textwidth]{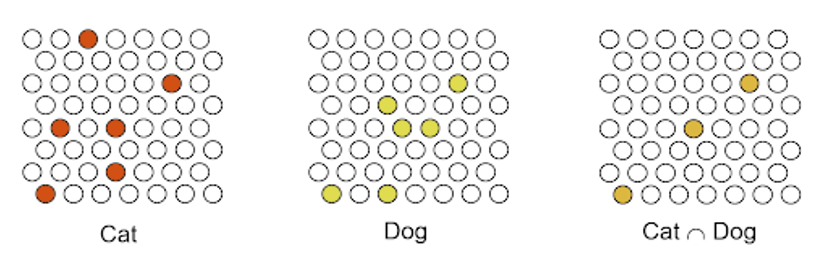}    
\caption{Visualizing sparse population coding for neuronal networks. Each dot represents a neuron, and colored dots show neurons which respond to seeing a cat or a dog.}
\label{fig:codingillus}
\end{figure}

While there has been much work on fitting Hawkes processes in the literature (see \cite{reinhart2018review} for a recent survey), change-point detection for Hawkes processes is left an important topic with only very little attention and much less studied. 
In \cite{wang2020detecting}, the offline change-point detection problem for high-dimensional Hawkes processes was studied, and the goal is to estimate (multiple) change-points. 
In \cite{rambaldi2018detection}, a model selection scheme was proposed to identify the presence of exogenous events that increase the intensity of the Hawkes process for a given time period. 
A cumulant-based multi-resolution segmentation algorithm was proposed in \cite{zhou2020fast} to find the optimal partition of the nonstationary Hawkes process into several non-overlapping segments. 
On the contrary, we focus on the \textit{sequential} detection problem, which aims to detect the change as quickly as possible. Online change-point detection for Hawkes processes was considered in \cite{li2017detecting}, where the generalized likelihood ratio (GLR) test was used to detect the change with unknown post-change parameters. In that work, the expectation-maximization (EM) algorithm was used to estimate unknown post-change parameters, which does not allow for an efficient recursive implementation and could be time-consuming. For the target problem of neuronal network detection, we would like to detect the change in real-time from streaming data, utilizing a more computationally efficient procedure.

In this paper, we present a novel CUSUM procedure for sequential change detection in Hawkes processes. The recursive CUSUM is based on a log-likelihood ratio statistic, which is further modified to improve computational efficiency with the practical consideration of removing historical data with long lags. The new CUSUM procedure is computationally and memory efficient as a recursive procedure, which is crucial for its online implementation in practice, such as sensor network and social network monitoring problems. We study the theoretical properties of this new CUSUM procedure, including an analysis of its average run length (ARL) and expected detection delay (EDD). We then compare the proposed CUSUM procedure with existing change detection algorithms based on the GLR statistic \citep{li2017detecting} and the score statistic in a comprehensive simulation study. Finally, we apply our method to the aforementioned motivating neuroengineering problem on population code change-detection for biological neural networks. Numerical results show that the proposed CUSUM procedure outperforms existing alternative methods.

The rest of the paper is organized as follows. Section~\ref{sec:prelim} introduces the basics for Hawkes processes. Section~\ref{sec:method} sets up the change-point detection problem, outlines the proposed CUSUM procedure, discusses algorithmic developments for computational and memory efficiency, and presents its theoretical properties. Section~\ref{sec:alternative} discusses some alternative detection methods. Sections~\ref{sec:numerical_comparison} and \ref{sec:realdata} compares the proposed CUSUM approach with existing methods for a simulation study and a real-world application using neural spike train data. Section~\ref{sec:conclusion} concludes the paper with some discussions.

\section{Preliminaries}\label{sec:prelim}

We first provide some background on point processes and Hawkes processes, which will be used in later sections.

A temporal point process is a random process whose realization consists of a sequence of discrete events occurring at times $\{t_i, i=1,2,\ldots\}$, with $t_i \in \mathbb{R}^{+}$. Let the history $\mathcal{H}_{t^-}$ be the sequence of times of events $\{t_1, t_2, \dots, t_n\}$ up to but {\it not} including time $t$.  
Let $N_t$ represents the number of events before time $t$, then $N_t$ is a counting process which can be defined as:
$
dN_t = \sum_{t_i\in \mathcal{H}_t}\delta(t-t_i)dt,
$
where $\delta$ is the Dirac function.
The sequence of discrete event times $\{t_i, i=1,2,\ldots\}$ can be regarded as when the counting process $N_t$ has jumped. 

A point process can be characterized by its conditional intensity function, denoted as $\lambda(t)$. This conditional intensity function is also known as the hazard function \citep{rasmussen2011temporal}, and is defined as
$
\lambda(t) = f^*(t)/(1-F^*(t)).
$
Here, $f^*(t)$ is the probability density function of the next event time conditional on the past, and $F^*(t) = \mathbb{P}\{t_{n+1}<t | \mathcal{H}_{t^-}\}$ is the associated conditional cumulative distribution function capturing the probability of the $(n+1)$-th event happening before time $t$. Thus if we consider a small time interval $[t,t+dt)$, we have
\begin{equation*}
\lambda(t) dt =\frac{f^*(t)dt}{1-F^*(t)} =  \frac{\mathbb{P}(t_{n+1}\in[t,t+dt))}{\mathbb{P}(t_{n+1}\geq t)} = \mathbb{P}\{t_{n+1}\in[t,t+dt)|\mathcal H_{t^-}\}.
\end{equation*}

\subsection{One-Dimensional Point Processes}

For one-dimensional Hawkes process, the intensity function takes the form \citep{hawkes1971spectra}: 
\begin{equation}
\lambda(t) = \mu(t) + \alpha  \int_{0}^{t} \varphi(t-\tau) dN_{\tau},
\end{equation}
where $\mu(t)$  is the base intensity, $\alpha$ is the influence parameter, and $\varphi(t)$ is a normalized kernel function satisfying $\int \varphi(t)dt=1$. A commonly used kernel function is the exponential kernel $\varphi(t) = \beta e^{-\beta t}$ with $\beta>0$. We assume $0\leq\alpha<1$ to ensure a stationary process.

Given event times $\{t_1,t_2,\ldots, t_n\}$ which happened before a given time $t<\infty$, the log-likelihood function for the Hawkes process can be written as follows (see \cite{daley2003introduction} for details):
\begin{align}\label{eq:likelihood}
\ell_t= \sum_{i =1 }^n \log \left[  \mu(t_i)+ \alpha\sum_{t_j < t_i} \varphi(t_i-t_j) \right]
- \int_0^t \mu(s) ds - \alpha \sum_{i=1}^n  \int_{t_i}^t \varphi(s-t_i) ds.
\end{align}
In case of the exponential kernel $\varphi(t) = \beta e^{-\beta t}$ and a constant base intensity $\mu$, \eqref{eq:likelihood} reads 
$$\ell_t = \sum_{i= 1}^n    \log \left[ \mu+ \alpha \sum_{t_j < t_i} \beta e^{-\beta (t_i-t_j)} \right]
-  \mu t -\alpha \sum_{i=1}^n  \left[ 1- e^{-\beta(t-t_i)}\right].$$
 As we will see in the following, this log-likelihood plays a key role in sequential change detection procedures.

\subsection{Network Point Processes}\label{sec:multivariate_Hawkes}

The multivariate Hawkes process on a network with $D$ nodes is represented by a series of event times together with their location $\{(t_i,u_i),i=1,2,\ldots\}$, where $t_i \in \mathbb{R}^{+}$ is the event time and $u_i\in [D]$ is the node on which the $i$-th event occurs. Here we use $[D]$ to represent the set $\{1,\cdots,D\}$. The intensity function for node $i$ at time $t$ is
$$
\lambda_i(t) = \mu_i(t) + \sum_{j\in[D]}\alpha_{ij}\int_0^t\varphi_{ij}(t-s)dN_s^j,
$$
where $\mu_i(t)$ is the base intensity at node $i$, $\alpha_{ij}$ is the influence parameter from node $j$ to node $i$, $\varphi_{ij}(t)$ is a normalized kernel function, and $N_t^j$ is a counting process on node $j$:
$
dN_t^j = \sum_{k:t_k<t,u_k = j}\delta(t-t_k)dt.
$
The log-likelihood function for the network setting up to time $t$ is given by:
\begin{equation}
  \ell_t(A) = -\sum_{i\in[D]}\int_0^t \lambda_i(s)ds + \sum_{i\in[D]}\int_0^t \log(\lambda_i(s))dN_s^i,
  \label{eq:log-likelihood}
\end{equation}
where $A = (\alpha_{ij})_{i,j\in[D]}\in\mathbb R^{D\times D}$ is the matrix representation for the influence parameters.

The log-likelihood expression in Equation \eqref{eq:log-likelihood} reveals a useful property which we later exploit for distributed change-point detection. Note that this log-likelihood can be decoupled as the summation over $D$ nodes, in that it consists of the sum of the log-likelihood at each node. Furthermore, the intensity function $\lambda_i(\cdot)$ only involves events observed on the neighbors of $i$, i.e., the nodes which influence node $i$. This property allows us to develop a \textit{distributed} change-point detection procedure, where each node can compute their likelihood in parallel, and only needs to communicate with neighboring nodes and not over the entire network (assuming such neighborhood information is known beforehand).

\section{Proposed CUSUM Detection Framework}\label{sec:method}
\subsection{Problem Set-up}\label{sec:setup}

The problem of change-point detection for Hawkes networks can be set-up as follows. Assume there exists a true change-point time $\kappa > 0$, and the event data follows one point process before the change-point and follows another point process afterward. We consider in this work two specific cases: (i) the null (pre-change) point process is a Poisson process, whereas the alternative (post-change) point process is a Hawkes point process; (ii) the null point process is a Hawkes point process, whereas the alternative point process is a different Hawkes point process, e.g., the influence parameter $A$ has been shifted. Note that the first scenario can be seen as a specific case of the second, since a Poisson process can be viewed as a specific Hawkes process with influence parameters set as 0.  

Consider now a hypothesis test for detecting temporal pattern shifts in the Hawkes process. Assuming the Hawkes process is stationary and the change-point $\kappa$ is an {\it unknown} variable, this test can be formulated as: 
\begin{equation}
\label{hypothesis test}
\begin{array}{ll}
 H_0: &\lambda_i^* (s)=\mu_i(s) + \sum_{j\in[D]}\alpha_{ij,0}  \int_0^{s} \varphi_{ij}(s-v)  dN_{v}^j; \quad i\in[D], s \geq 0,\\
 H_1: &  \lambda_i^* (s)=\mu_i(s) + \sum_{j\in[D]}\alpha_{ij,0}  \int_0^{s} \varphi_{ij}(s-v)  dN_{v}^j; \quad i\in[D], 0 \leq s \leq\kappa,\\
&  \lambda_i^* (s)=\mu_i(s) + \sum_{j\in[D]}\alpha_{ij,1}  \int_\kappa^{s} \varphi_{ij}(s-v)  dN_{v}^j; \quad i\in[D], s > \kappa.
\end{array}
\end{equation}
Here, $\lambda_i^*(s)$ denotes the true intensity for node $i$ at time $s$. The pre-change parameters $\{\alpha_{ij,0}\}_{i,j\in[D]}$ can typically be elicited from prior knowledge of the process or estimated from reference data. The post-change parameters $\{\alpha_{ij,1}\}_{i,j\in[D]}$ are known in some scenarios, but more often it corresponds to an unexpected anomaly and we may not have enough data to estimate this in advance. Alternatively, we can treat the post-change parameters as the targeted smallest change to be detected. A change detection procedure resolves the two hypotheses using a stopping time $T$, which is a function of the event sequence, as explained next.

\subsection{A Recursive CUSUM Statistic}\label{sec:cusum}

We now present the cumulative sum (CUSUM) statistics based on the log-likelihood ratio. The CUSUM procedure was first proposed in \cite{page-biometrica-1954}, assuming both pre- and post-change parameters are provided (or estimated). The CUSUM is known to be computationally efficient since it can be computed recursively. 
The CUSUM procedure is most commonly defined for i.i.d. observations, but there is much recent development in extending CUSUM for non i.i.d. observations \citep{tartakovsky2014sequential,reviewCP2021}.

We first define the log-likelihood ratio function which will be used as the main building block of our procedure. For a hypothesized change-point $\tau$, the log-likelihood ratio of the model \eqref{hypothesis test} up to time $t$ can be derived as:
\begin{align}
\ell_{t, \tau} = \sum_{i\in[D]}\int_{\tau}^t\log \left( \frac{\lambda_{i,\tau}(s)}{\lambda_{i,\infty}(s)} \right) dN_s^i  - \sum_{i\in[D]}\int_{\tau}^t (\lambda_{i,\tau}(s)-\lambda_{i,\infty}(s)) ds,
\label{eq:log-likelihood_ratio}
\end{align}
where $$
\lambda_{i,\tau}(t) = \begin{cases}
\mu_i(t) + \sum_{j\in[D]}\alpha_{ij,1}\int_\tau^t\varphi_{ij}(t-s)dN_s^j,& t> \tau,\\
\lambda_{i,\infty}(t), & 0\leq t\leq\tau,
\end{cases}
$$
is the intensity for node $i$ if the change-point happens at $\tau$, and 
$
\lambda_{i,\infty}(t) = \mu_i(t) + \sum_{j\in[D]}\alpha_{ij,0}\int_0^t\varphi_{ij}(t-s)dN_s^j 
$
is the intensity under the null hypothesis. Here, $\infty$ is used to indicate that the event that the change never happens.


Given assumed post-change parameters, $\{\alpha_{ij,1}\}_{i,j\in[D]}$, the stopping time for CUSUM is given by
\begin{align}\label{eq:cusum}
\Tc = \inf \{t:\, \sup_{\tau < t}  \ell_{t, \tau} >b \},
\end{align}
where $\ell_{t,\tau}$ is the log-likelihood ratio statistic defined in Equation \eqref{eq:log-likelihood_ratio}, and $b >0$ is a pre-specified threshold. 
The procedure stops when the log-likelihood ratio from some hypothesized change-point $\tau$ exceeds threshold $b$. 

In contrast to the original CUSUM procedure \citep{page-biometrica-1954} where the samples are taken in a discrete-time fashion, here the CUSUM statistic is continuous-time and has memory. In particular, due to the memory of the Hawkes process, the observations are {\it non-i.i.d.} and have complex temporal dependence. Because of this dependence, the simple recursive approach for standard CUSUM does not extend to the current (more complex) Hawkes process setting, and further developments are needed.

To derive an computationally efficient recursive algorithm for CUSUM in the network Hawkes process, we start with a lemma for the log-likelihood ratio $\ell_{t, \tau}$. This lemma shows that, although the supremum of the log-likelihood ratio statistic over the unknown change-point appears to be on a continuum, it will be obtained at the observed event times.
\begin{lemma}\label{lem:maximizer}
Given the event times $\{t_i,i=1,2,\ldots\}$, for any fixed $t$ and $k$, $t>t_k$, it follows that
$
\sup\limits_{t_{k}< \tau \leq \min\{t_{k+1},t\}} \ell_{t,\tau} =\lim_{\tau\to t_k^+}\ell_{t,\tau}=:\ell_{t,t_k^+},
$
and
$
\sup\limits_{0\leq \tau \leq t_1} \ell_{t,\tau} = \ell_{t,0}.
$
\end{lemma}
\noindent The proof of this lemma is provided in the supplementary files. Lemma~\ref{lem:maximizer} says that we only need to consider the values of the log-likelihood evaluated as the past event times, rather than a continuum of possible values for $\tau$. As we show later, this will greatly simplify the computation of the log-likelihood ratio statistic. 

For computational efficiency, we can further simplify the calculation in \eqref{eq:cusum} (which involves $\sup_{\tau<t}\ell_{t,\tau}$) by considering $t$ on a discretized grid with a pre-specified grid size $\gamma>0$. In this case, we would only need to calculate the detection statistic $\sup_{\tau<n\gamma}\ell_{n\gamma,\tau}$ for $n\in\mathbb Z$.

%
Finally, with the discretization for both $\tau$ and $t$, the log-likelihood ratio $\ell_{n\gamma,t_k^+}$ and $\ell_{(n+1)\gamma,t_k^+}$ have the following relationship, given $n\gamma\geq t_k$,
\begin{equation}
\ell_{(n+1)\gamma,t_k^+} = \ell_{n\gamma,t_k^+} + \sum_{i\in [D]}\int_{n\gamma}^{(n+1)\gamma}\log\left(\frac{\lambda_{i,t_k^+}(s)}{\lambda_{i,\infty}(s)}\right)dN_s^i - \sum_{i\in[D]}\int_{n\gamma}^{(n+1)\gamma}(\lambda_{i,t_k^+}(s)-\lambda_{i,\infty}(s))ds.
\label{eq:update_ell}
\end{equation}
Equation \eqref{eq:update_ell} provides a recursive procedure for computing the log-likelihood ratios $\ell_{n\gamma,t_k^+}$, as long as the one-dimensional integrals can be evaluated or approximated numerically. If we have additional access to the cumulative kernels
$
\Phi_{ij}(t) = \int_0^t\varphi_{ij}(s)\mathbbm{1}(s\geq 0)ds,
$
where $\mathbbm 1(\cdot)$ is the indicator function,
this recursion can be computed without numerical integration as follows:
\begin{align}
\ell_{(n+1)\gamma,t_k^+} = \ell_{n\gamma,t_k^+} +&\ \sum_{i\in [D]}\int_{n\gamma}^{(n+1)\gamma}\log\left(\frac{\lambda_{i,t_k^+}(s)}{\lambda_{i,\infty}(s)}\right)dN_s^i \nonumber\\
+&\ \sum_{i,j\in[D]}\int_{0}^{(n+1)\gamma}\alpha_{ij,0}(\Phi_{ij}((n+1)\gamma-s)-\Phi_{ij}(n\gamma-s))dN_s^j \nonumber\\
-&\ \sum_{i,j\in[D]}\int_{t_k^{+}}^{(n+1)\gamma}\alpha_{ij,1}(\Phi_{ij}((n+1)\gamma-s)-\Phi_{ij}(n\gamma-s))dN_s^j.
\label{eq:update_ell_no_integral}
\end{align}

\begin{algorithm}[htbp!]
\SetAlgoLined
\textbf{Input:} event times 
$\{(t_i,u_i),i=1,2,\ldots\}$; pre-change parameters $\{\alpha_{ij,0}\}_{i,j\in[D]}$; post-change parameters $\{\alpha_{ij,1}\}_{i,j\in[D]}$; threshold $b$; grid size $\gamma$\;
Initialization: $n\leftarrow 0$,\,$S_0\leftarrow 0$,\,$\ell_{0,0} \leftarrow 0$\;
\While{$S_{n\gamma} < b$}{
$n \leftarrow n+1$\;
\For{$\tau \in\{t_k^+ : t_k<n\gamma\}\cup \{0\}$}{
\eIf{$\tau < (n-1)\gamma$}{
calculate $\ell_{n\gamma,\tau}$ using \eqref{eq:update_ell} or \eqref{eq:update_ell_no_integral}\;}{ 
calculate $\ell_{n\gamma,\tau}$ using \eqref{eq:log-likelihood_ratio}\;
}
}
$S_{n\gamma} \leftarrow \max_\tau\ell_{n\gamma,\tau}$\;
\If{$S_{n\gamma} > b$}{
\textbf{Output}: stopping time $T_{\scriptscriptstyle \text{\mdseries C}} \leftarrow n\gamma,$\, $\hat\tau \leftarrow \arg\max_\tau\ell_{n\gamma,\tau}$\;}
}
\caption{CUSUM for network Hawkes processes}\label{alg:CUSUM}
\end{algorithm}

Algorithm~\ref{alg:CUSUM} summarizes the key steps in the proposed CUSUM procedure. We provide a further remark on the choice of the grid size $\gamma >0$. Note that different choices of $\gamma$ corresponds to different updating frequencies for the CUSUM statistics, hence $\gamma$ is an important parameter for the algorithm. There is a performance trade-off in choosing the parameter $\gamma$: a very large choice of $\gamma$ may result in a large detection delay, whereas a very small $\gamma$ may leads to unnecessary computational complexity. The effect of $\gamma$ on the algorithm is investigated further in numerical studies, and it would be interesting to develop a method for choosing $\gamma$ adaptively.

\subsection{Modification for Memory Efficiency} 

We now present a modification of Algorithm \ref{alg:CUSUM} to improve memory efficiency of the procedure. Note that the ``exact'' CUSUM algorithm (Algorithm \ref{alg:CUSUM}) requires keeping track of the entire history of events, since all past events influence the intensity function. However, in practice, events which happened long ago will have little bearing on the current intensity, since its mutually-exciting property diminishes over time. One way to improve memory efficiency is to simply remove such events, since their influence on the present and the future (and thereby the performance of the method) would be small.

Consider the following {\it truncated kernel} with a width $B>0$:
\[
\widetilde \varphi_{ij}(t)=\begin{cases}
\varphi_{ij}(t),& t\leq B,\\
0,&t>B.
\end{cases}
\] 
Under the truncated kernel, an event has no influence over the whole process after $B$ into the future, and we only need to keep events during $[t-B,t]$ in our memory for computation. 
With the truncated kernel $\widetilde \varphi_{ij}$, the intensity for node $i$ can then be approximated by
\begin{equation}
\widetilde\lambda_{i,\tau}(t) = \begin{cases}
\mu_i(t) +\sum_{j\in[D]}\alpha_{ij,1} \int_{t-B}^t \varphi_{ij}(t-s)dN_s^j,& \tau<t-B,\\
\mu_i(t) +\sum_{j\in[D]}\alpha_{ij,1} \int_{\tau}^t \varphi_{ij}(t-s)dN_s^j,& t-B\leq \tau\leq t,\\
\mu_i(t) +\sum_{j\in[D]}\alpha_{ij,0} \int_{t-B}^t \varphi_{ij}(t-s)dN_s^j,& \tau> t.
\end{cases}
\label{eq:lambda_truncated}
\end{equation}
Here for all $\tau\geq 0$, $\widetilde\lambda_{i,\tau}(t)$ only depends on event data during $[t-B,t]$. Moreover, the intensity for $\tau<t-B$ does not depend on $\tau$, which enables us to update the log-likelihood ratio recursively for small $\tau$. If we also have access to the cumulative kernels $\Phi_{ij}, i,j\in [D],$ the recursion step in Equation \eqref{eq:update_ell_no_integral} can be approximated by
\begin{align}
\ell_{(n+1)\gamma,t_k^+} = \ell_{n\gamma,t_k^+} +&\ \sum_{i\in [D]}\int_{n\gamma}^{(n+1)\gamma}\log\left(\frac{\widetilde\lambda_{i,t_k^+}(s)}{\widetilde\lambda_{i,\infty}(s)}\right)dN_s^i \nonumber\\
+&\ \sum_{i,j\in[D]}\int_{n\gamma-B}^{(n+1)\gamma}\alpha_{ij,0}(\widetilde\Phi_{ij}((n+1)\gamma-s)-\widetilde\Phi_{ij}(n\gamma-s))dN_s^j \nonumber\\
-&\ \sum_{i,j\in[D]}\int_{\max\{t_k^+,n\gamma -B\}}^{(n+1)\gamma}\alpha_{ij,1}(\widetilde\Phi_{ij}((n+1)\gamma-s)-\widetilde\Phi_{ij}(n\gamma-s))dN_s^j,
\label{eq:update_ell_truncated_no_integral}
\end{align}
where $\widetilde\Phi_{ij}(t) = \Phi_{ij}(\min\{t,B\})$, and the summation is taken only for event times during $[n\gamma -B,(n+1)\gamma]$. 

Figure \ref{fig:algo_plot} shows an illustration of the memory-efficient CUSUM procedure. The details are summarized in Algorithm~\ref{alg:CUSUM2}. 

The computing and memory resources required for this procedure depend on both the network size and the number of events observed while monitoring. Each time we update the CUSUM statistics from $t = (n-1)\gamma$ to $n\gamma$, we track the log-likelihood ratios $\ell_{(n-1)\gamma,\cdot}$ for potential change-points $\tau$ that are event times from $(n-2)\gamma - B$ to $n\gamma$, along with a summary of the log-likelihood ratios for $\tau<(n-2)\gamma - B$. For each $\ell_{\cdot,\tau}$, when updating from $t = (n-1)\gamma$ to $t = n\gamma$ by \eqref{eq:update_ell_truncated_no_integral}, each $\widetilde\lambda_{\cdot,\cdot}(\cdot)$ is calculated using at most $N_{n\gamma} - N_{(n-1)\gamma -B}$ events, and the integral over counting measure is the summation over at most $(N_{n\gamma} - N_{(n-1)\gamma -B})D$ terms. Overall, the computation complexity for one update is $O((N_{n\gamma} - N_{(n-2)\gamma-B})^3D)$. Under the stability condition $\|A\|<1$ (where $\|\cdot\|$ is the spectral norm), a multi-dimensional Hawkes process can be shown to have a finite third-order moment and is ergodic \citep{achab2017uncovering}. In this case, the computation complexity of the memory-efficient CUSUM is linear in the time $t$.

Note that the dependency of the computation complexity on network size $D$ can be eliminated if, for each $j\in[D]$, $(\widetilde\Phi_{ij})_{i\in[D]}$ are identical. When adding the term on the second and third row in \eqref{eq:update_ell_truncated_no_integral}, the summation over $i$ can be pre-computed by saving $\sum_{i}\alpha_{ij,0}$, $\sum_{i}\alpha_{ij,1}$ for each $j\in[D]$. Thus it only takes 
$O((N_{n\gamma} - N_{(n-2)\gamma-B})^3)$ steps to perform the $n$-th update.

Regarding memory usage of the procedure, note that for the updates up until time $(n-1)\gamma$, we only need to keep track of event data with an occurrence time after $(n-1)\gamma-B$. Hence, the memory usage for the $n$-th update (apart from loading the network parameters) is $O(N_{n\gamma} - N_{(n-1)\gamma-B})$, the average is a constant with respect to time $t$.

\begin{figure}[htbp!]
    \centering
    \includegraphics[width = 0.8\textwidth]{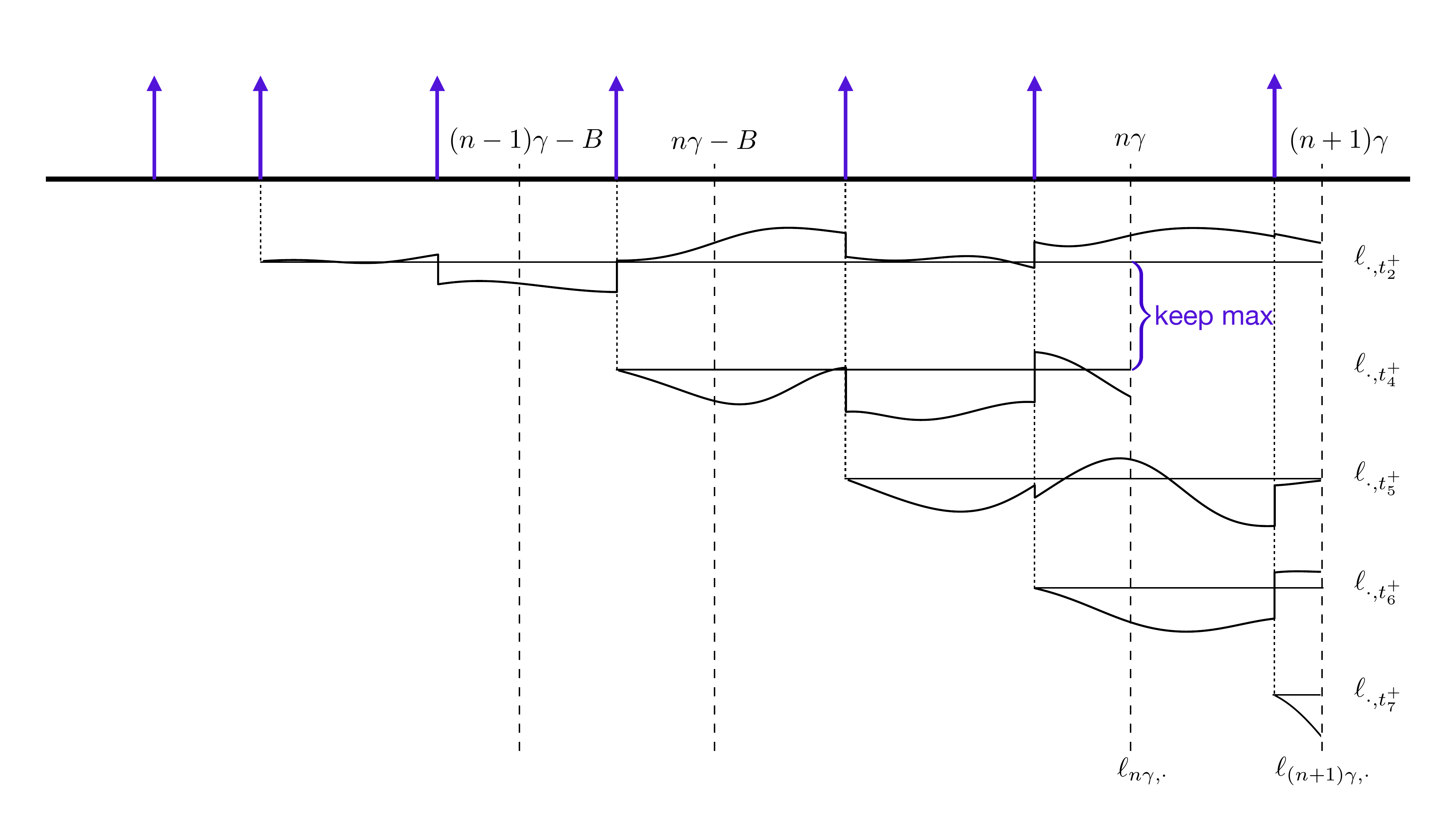}
    \caption{Illustration of the memory-efficient CUSUM algorithm. At time $t = n\gamma$, only one of the log-likelihood ratios $\ell_{n\gamma,\tau}$ for potential change-points $\tau<(n-1)\gamma - B$ is kept in memory (in this example, the one immediately after the second event). When updating the statistics to $t=(n+1)\gamma$, all log-likelihood ratios $\ell_{n\gamma,\tau}$ for $\tau\leq n\gamma - B$ are compared and again only the largest one is kept for future updates. }
    \label{fig:algo_plot}
\end{figure}

{\scriptsize
\begin{algorithm}[t!]
\SetAlgoLined
\textbf{Input:} event times 
$\{(t_i,u_i),i=1,2,\ldots\}$; pre-change parameters $\{\alpha_{ij,0}\}_{i,j\in[D]}$; post-change parameters $\{\alpha_{ij,1}\}_{i,j\in[D]}$; threshold $b$; grid size $\gamma$\;
Initialization: $n\leftarrow 0$,\,$S_0\leftarrow 0$,\,$\ell_{0,0} \leftarrow 0$,\,$\hat\tau \leftarrow 0$\;
\While{$S_{n\gamma} < b$}{
$n\leftarrow n+1$\;
\For{$\tau\in\{t_k^+ : (n-1)\gamma\leq t_k<n\gamma$\}}{
calculate $\ell_{n\gamma,\tau}$ using \eqref{eq:log-likelihood_ratio} with \eqref{eq:lambda_truncated}\;
}
\For{$\tau\in\{t_k^+:(n-1)\gamma-B\leq t_k<(n-1)\gamma$}{
calculate $\ell_{n\gamma,\tau}$ using \eqref{eq:update_ell} with \eqref{eq:lambda_truncated} or \eqref{eq:update_ell_truncated_no_integral}\;
}
\For{$\tau\in\{t_k^+: (n-2)\gamma -B\leq t_k<(n-1)\gamma-B\}$}{
\If{$\ell_{(n-1)\gamma,\tau}>\ell_{(n-1)\gamma,\hat\tau}$}{
$\hat\tau\leftarrow \tau$\;
}
}
update $\ell_{n\gamma,\hat\tau}$ from $\ell_{(n-1)\gamma,\hat\tau}$ using \eqref{eq:update_ell} with \eqref{eq:lambda_truncated} or \eqref{eq:update_ell_truncated_no_integral}\;

$S_{n\gamma} \leftarrow \max_\tau \ell_{n\gamma,\tau}$\;
\If{$S_{n\gamma} > b$}{
\textbf{Output}: stopping time $T_{\scriptscriptstyle \text{\mdseries C}} \leftarrow n\gamma$,\,$\hat\tau\leftarrow \arg\max_\tau\ell_{n\gamma,\tau}$\;
}
}
\caption{Memory-efficient CUSUM for network Hawkes processes}\label{alg:CUSUM2}
\end{algorithm}}

We illustrate the effect of the truncation width $B$ using a numerical example. The model is described in Section~\ref{sec:numerical_comparison}, where the kernel functions are all exponential 
$\varphi_{ij}(t) = \beta e^{-\beta t},t\geq 0,~ \forall i,j\in[D]$
with $\beta = 1$. Figure~\ref{fig:effect_of_B} shows the comparison between CUSUM statistics with and without kernel truncation. Both CUSUM statistics with the truncated kernel have the same trend as the exact CUSUM statistic. When $B = 1/\beta$, $63.2\%$ of the cumulative influence (which corresponds to the integral of the truncated kernel since the complete influence kernel integrates to one) 
is preserved, and the truncated statistic deviates from the exact CUSUM, which may result in a false alarm. When $B = 2/\beta$, $86.5\%$ of the cumulative influence is preserved, and there appears to be little difference between the truncated and exact CUSUM statistics. 

{\color{black} \textbf{Performance.} We discuss the performance of the proposed procedure via two widely used metrics for sequential change-point detection: (i) Average Run Length (ARL), defined as the expected value of the stopping time when there is no change, i.e., $\mathbb E_\infty[T]$, where $\mathbb P_\infty$ is the probability measure on the sequence of event times when the change never occurs, and $\mathbb E_\infty$ is its corresponding expectation; (ii) Expected Detection Delay (EDD), defined as the expected delay between the stopping time and the true change-point. Two common definitions for EDD can be found in \cite{Lorden1971} and \cite{pollak1985optimal}, both of which consider the worst-case delay over all possible change-point values. In particular, if the true change-point is $\kappa$, then the EDD can be defined as $\mathbb E_\kappa[T-\kappa|T>\kappa]$, where $\mathbb P_\kappa$ denotes the probability measure on the observations when the change occurs at time $\kappa$, and $\mathbb E_\kappa$ denotes the corresponding expectation. The theoretical properties of the ARL and EDD and be found in the supplementary files in Section B.}

\begin{figure}[htbp!]
    \centering
    \includegraphics[width = 0.9 \textwidth]{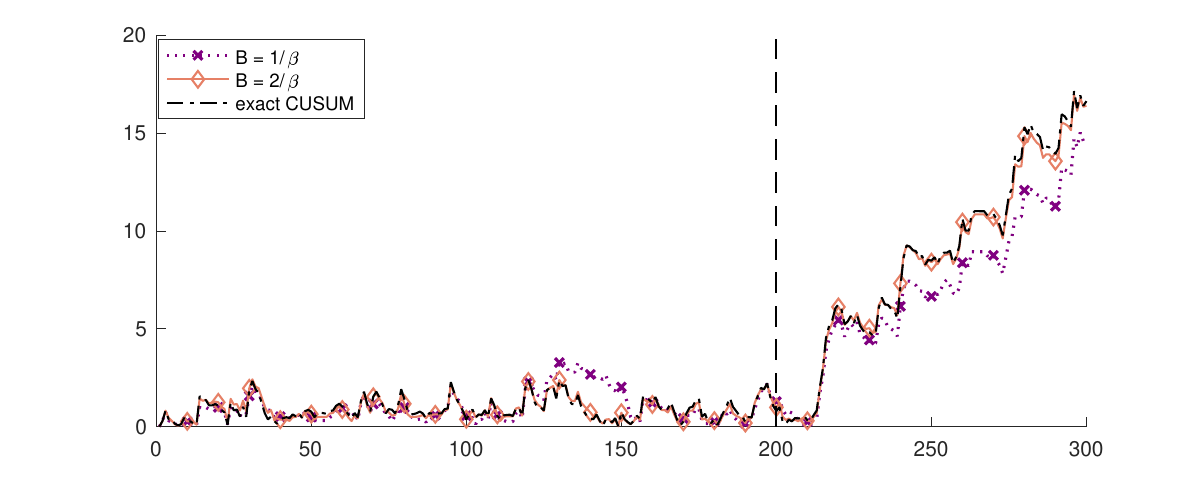}
    \caption{Comparison between CUSUM statistics with different truncation level $B$. The change-point occurs at time 200.}
    \label{fig:effect_of_B}
\end{figure}

\section{Alternative Detection Procedures}\label{sec:alternative}

In practical problems, the post-change parameters are not always known due to the lack of anomalous data (e.g., there could be various types of anomalies, and one may not know which anomaly to expect). This section discusses two alternate approaches to change-point detection on the Hawkes process: the score statistics and the GLR statistics. Neither method requires any knowledge of the post-change parameters.

\subsection{Score Statistics}

We consider the score statistics for constructing a detection procedure. The score statistic can detect any deviations from the null hypothesis \citep{xie2012spectrum}. It is particularly suitable for detecting small deviations (i.e., locally most efficient) and does not require estimating post-change parameters. The score function is defined as the derivative of the log-likelihood as in \eqref{eq:log-likelihood} over the parameters $\alpha_{ij,0}$, $i,j\in[D]$, on which we would like to detect the change. With $\boldsymbol \alpha_i = (\alpha_{ij})_{j\in[D]}\in\mathbb R^{D\times1}$, the score function on each node $i$ is
\begin{align}
\frac{\partial \ell_t}{\partial\boldsymbol\alpha_i} = -\int_0^t \frac{\partial\lambda_{i,\infty}(s)}{\partial \boldsymbol\alpha_i}ds + \int_0^t \lambda_{i,\infty}^{-1}(s)\frac{\partial\lambda_{i,\infty}(s)}{\partial \boldsymbol\alpha_{i}} dN_s^i,
\end{align}
since $\lambda_{i,\infty}$ only depends on the parameter $\boldsymbol\alpha_{i}$. Then the full score function becomes
$
\frac{\partial \ell_t}{\partial\text{vec}(A)} = \left(\frac{\partial\ell_t}{\partial \boldsymbol\alpha_1^T},\cdots,\frac{\partial\ell_t}{\partial \boldsymbol\alpha_D^T}\right)^T
$
where $A = (\alpha_{ij})_{ij\in[D]}$ and $\text{vec}(A) = (\boldsymbol\alpha_1^T,\cdots,\boldsymbol\alpha_D^T)^T$. In Theorem 3.4 of \cite{ogata1978asymptotic}, it is shown that the limiting distribution of the score function at the true parameter $A_0 = (\alpha_{ij,0})_{i,j\in[D]}$ is normally distributed, i.e.
$\frac{1}{\sqrt{t}}\frac{\partial\ell_t}{\partial \text{vec}(A)}\bigg|_{A = A_0} \to \mathcal N(0,I_0),
$
where $$
I_0 = \lim_{t\to\infty}\frac{1}{t}\mathbb E\left[\frac{\partial\ell_t}{\partial \text{vec}(A)}\frac{\partial\ell_t}{\partial \text{vec}^T(A)}\right]\bigg|_{A = A_0} = -\lim_{t\to\infty}\frac{1}{t}\mathbb E\left[\frac{\partial^2\ell_t}{\partial\text{vec}(A)\partial\text{vec}^T(A)}\right]\bigg|_{A = A_0}
$$ is the Fisher information matrix. Note that the Fisher information here is a diagonal block matrix, since for any $i,j,i',j'\in[D]$, $i\neq i'$, $\lambda_{i,\infty}$ is a constant with respect to $\alpha_{i',j'}$, and $\frac{\partial^2\ell_t}{\partial\alpha_{i,j}\partial\alpha_{i',j'}}=0.$ Each block of $I_0$ corresponds to $\boldsymbol\alpha_{i},i\in[D]$, the influence from all nodes to node $i$. The limiting distribution of the score function at the true parameter can then be shown to be
$
\frac{1}{t}\frac{\partial\ell_t}{\partial \text{vec}^T(A)}I_0^{-1}\frac{\partial\ell_t}{\partial \text{vec}(A)}\bigg|_{A = A_0} \to \chi^2_{D^2},
$
where $\chi^2_\nu$ is the chi-squared distribution with degrees-of-freedom $\nu$.

For change-point detection, we adopt the conventional sliding window approach, i.e., calculating the score statistic inside the sliding window $[t-w,t)$ for a suitably chosen window length $w$, and raise an alarm whenever the statistic exceeds the threshold $b$. The corresponding stopping time can be written as:
\[
\Ts = \inf\left\{t: \frac{1}{w}\frac{\partial\mathcal (\ell_t - \ell_{t-w})}{\partial\text{vec}^T(A)}I_0^{-1}\frac{\partial\mathcal (\ell_t - \ell_{t-w})}{\partial \text{vec}(A)}\bigg|_{A = A_0}  \geq b \right\}.
\]
With a sufficiently large window length, the score statistic under the null hypothesis should be around $D^2$, the expected value of the chi-squared random variable $\chi^2_{D^2}$. After the change-point, the expected score function at the pre-change parameters is no longer 0. We would expect the score statistic to be noticeably larger than $D^2$, and thus the change is detected. 

Like CUSUM statistics, a memory-efficiency problem arises for the score statistics, since the intensity $\lambda_{i,\infty}$ depends on the whole history of events. We can again replace $\lambda_{i,\infty}$ with $\widetilde\lambda_{i,\infty}$ using the aforementioned truncated kernels (a similar approach can also be used for the GLR statistics, discussed next). For practical reasons, we would also need to choose a grid size $\gamma$ and compute the score statistics only on the resulting grid. We discuss the relation between grid size and EDD for a fixed ARL later in Section~\ref{sec:numerical_comparison}.

\subsection{GLR Statistics}

When the post-change parameters are unknown, another way to perform change-point detection is via the generalized likelihood ratio (GLR) statistic. The idea is to find the parameters which best fit the data, then compare the likelihood ratio between the fitted parameters and the pre-change ones. This GLR statistics approach for multi-dimensional Hawkes processes was discussed in \cite{li2017detecting}. Using a sliding window of fixed length $w$, the log-likelihood ratio can be defined within each window as
$$
\ell_{t,t-w,\hat A} = \sum_{i\in[D]}\int_{t-w}^{t} \log\left(\frac{\hat\lambda_{i,t-w}(s)}{\lambda_{i,\infty}(s)}\right)dN_s^i - \int_{t-w}^t (\hat\lambda_{i,t-w}(s)-\lambda_{i,\infty}(s))ds,
$$
where $\hat\lambda_{i,t-w}$ is the intensity assuming $t-w$ is the change-point and $\hat A$ is the post-change parameter estimates. 
$
\hat\lambda_{i,t-w}(s) = \mu_i(s) +  \sum_{j\in[D]}\hat\alpha_{ij}\int_{t-w}^s \varphi_{ij}(s-v)dN_v^j,
$
and $\hat A = (\hat\alpha_{ij})_{i,j\in[D]}$ is the parameter that maximizes the likelihood in the current window $[t-w,t]$:
$$
\hat A = \arg\max_{A}\sum_{i\in[D]}\int_{t-w}^{t} \log\lambda_{i,t-w}(s)dN_s^i - \int_{t-w}^t \lambda_{i,t-w}(s)ds.
$$
A change is detected when the log-likelihood ratio exceeds certain threshold $b$:
$$
\Tg = \inf\{t:\ell_{t,t-w,\hat A}\geq b\}.
$$
For each window, a convex optimization problem is solved to find the maximum likelihood estimate $\hat A$ that best fits the data. However, this operation makes the GLR statistic computationally more expensive than CUSUM and the score statistic. To address this issue, we can use $\hat A$ from the previous window as an initialization for the gradient descent algorithm to find the MLE in the next step -- this ``warm-start'' may lead to faster convergence for finding the MLE. 

Compared with the score statistic, the GLR statistic is computationally more expensive, and it does not necessarily have better performance (which can be partly due to the optimization error in computing MLE), as shown in our example in Section~\ref{sec:numerical_comparison}. However, the GLR statistic is numerically more stable than the score statistic, especially for large networks. The reason is that we usually may not estimate the Fisher information with high accuracy. Even provided with the exact pre-change parameter $A_0$, there is no close-form solution for the Fisher information, and it can only be estimated by simulation or real data. The score statistic also involves inverting the Fisher information, which can suffer from a high condition number and numerical instability.

\section{Numerical Experiments}
\label{sec:numerical_comparison}

In this section, we compare several change-point detection procedures using simulated examples on a small network with 8 nodes and a larger one with 100 nodes. The small network is shown in Figure~\ref{fig:model}. The base intensity is proportional to the size of the node ranging from 0.5 to 1, and the edges indicate the asymmetrical influences between nodes. The edges in black are the pre-change parameters, while the edges in orange show a topological change between nodes 1,2, and 3. There are two emerging edges after the change-point, while all other edges remain the same. For the 100-node network, the background intensity is set to be 0.05 for all nodes. All nodes work independently in the pre-change scenario, while the post-change network consists of 200 directed edges with weight 0.2 chosen at random. 

Throughout this section, we use the exponential decaying kernel $\varphi_{ij}(t) = e^{-t},\forall i,j$ to generate event data. The kernel functions are truncated at $B = 5$ to leverage the computational and memory efficient procedures in Section \ref{sec:cusum}. The update rate is set at $\gamma = 5$.

Figure~\ref{fig:stats} visualizes the CUSUM, GLR, and score statistics on the same sequence of events for both networks. As expected, CUSUM grows steadily larger after the change-point for both network settings. The differenced CUSUM, GLR and score statistics show similar post-change fluctuations, as they can be understood as various measurements of how the process within the sliding window differs from the pre-change scenario. The proposed CUSUM procedure appears to be the least noisy before the change-point, which may be another explanation of why CUSUM has the best performance (see later), apart from its cumulative nature.

\begin{figure}[htbp!]
\centering
\includegraphics[width=0.35\textwidth]{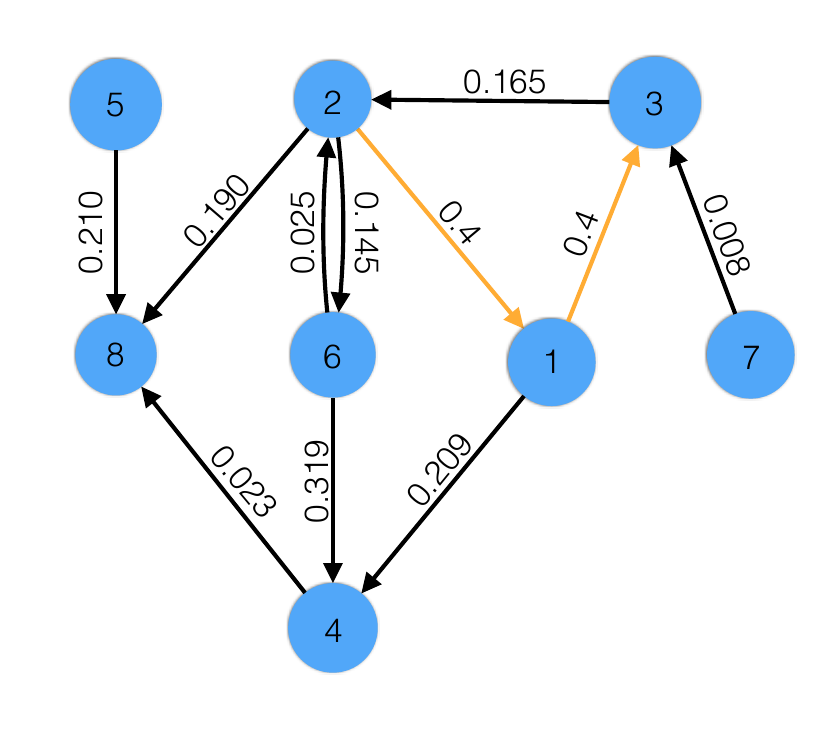}
\caption{Visualizing the simulated network Hawkes model. Here, the background intensity is proportional to the size of each node, with edges indicating directed influences between nodes. Orange edges show the topological change after the change-point.}
\label{fig:model}
\end{figure}

\begin{figure}[htbp!]
\centering
\begin{tabular}{cccc}
\includegraphics[width=0.22\textwidth]{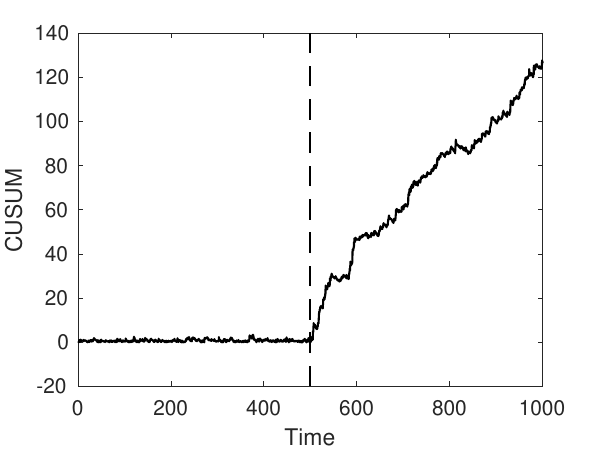} & \includegraphics[width=0.22\textwidth]{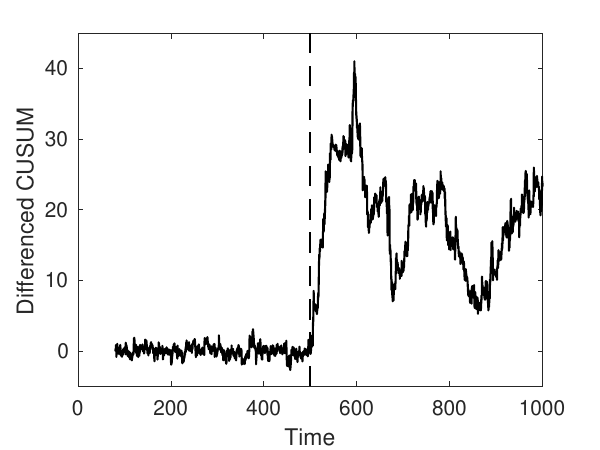} & \includegraphics[width=0.22\textwidth]{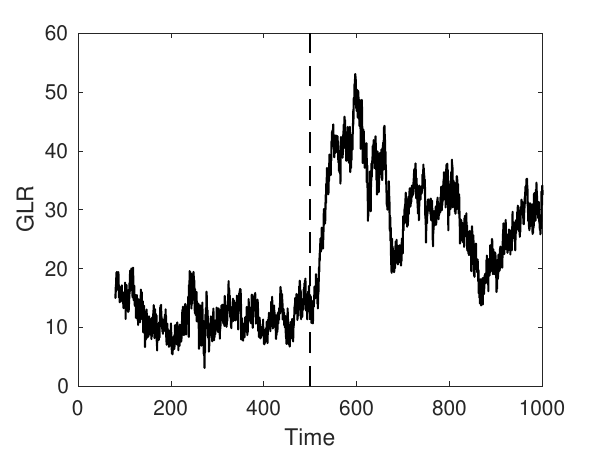}&\includegraphics[width=0.22\textwidth]{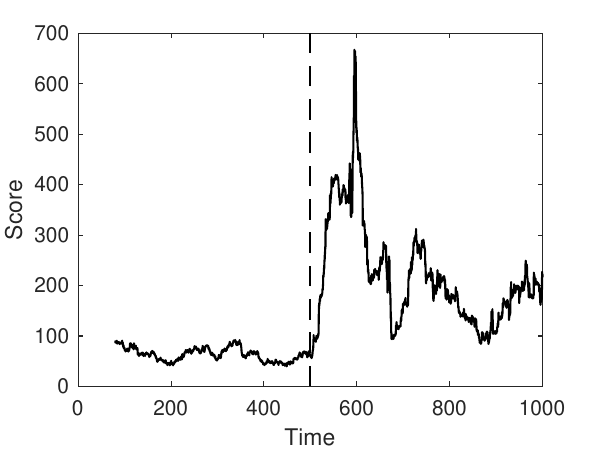}\\
\includegraphics[width=0.22\textwidth]{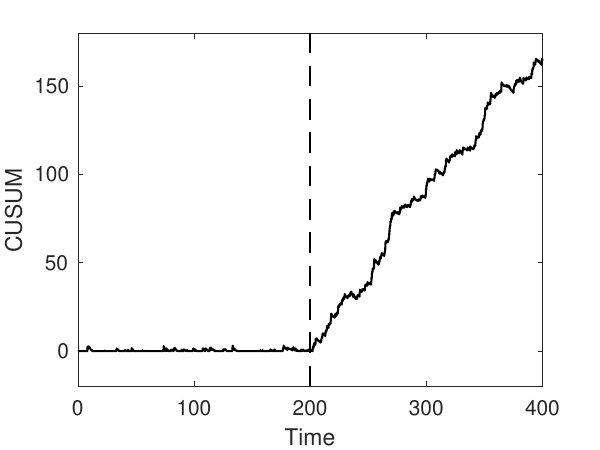} &\includegraphics[width=0.22\textwidth]{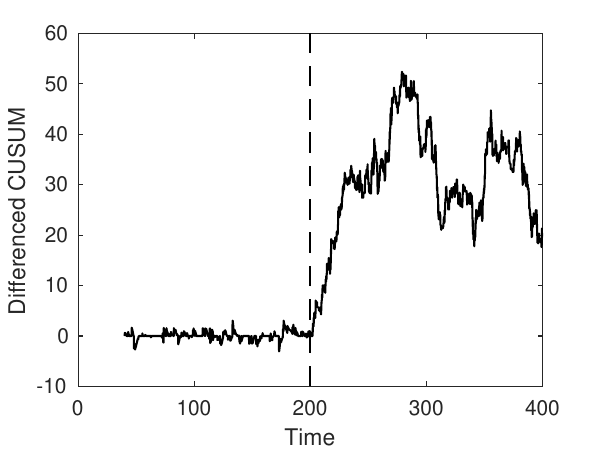} & \includegraphics[width=0.22\textwidth]{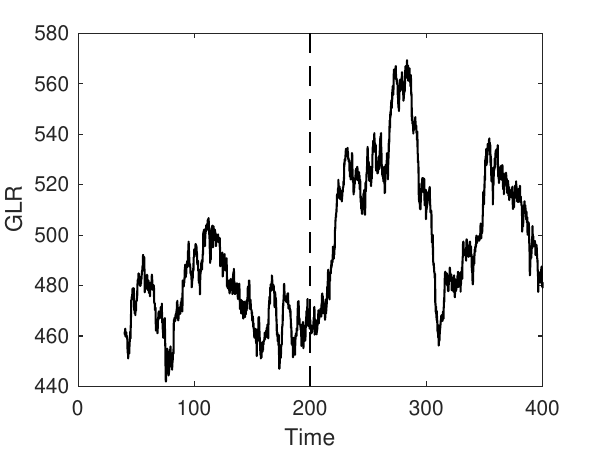}&\includegraphics[width=0.22\textwidth]{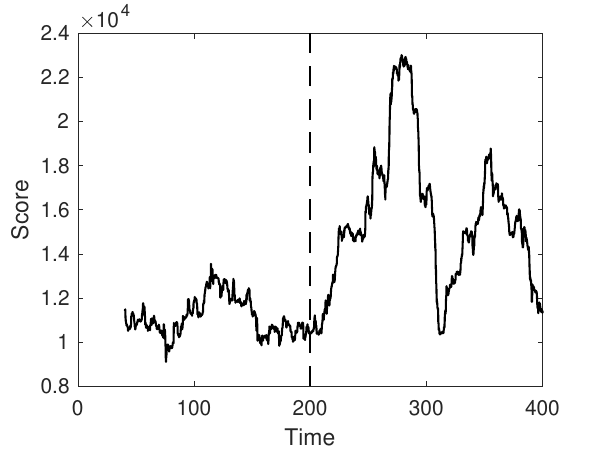}\\
(a)&(b)&(c)&(d)
\end{tabular}
\caption{Examples of detection statistics. The first row is generated on the small network of size 8, with window length $w = 80$ for both GLR and score statistics. The second row is generated on the large network of size 100, with window length $w = 40$. Both rows use the same sequence of events, and the window length is chosen numerically to optimize performance (see discussions in \ref{subsec:performance_comparison}). The vertical dashed line is the time of the true change-point $\kappa$. Column (b) shows the difference of CUSUM between time $t$ and $t-w$.}
\label{fig:stats}
\end{figure}



\subsection{Performance Comparison} 
\label{subsec:performance_comparison}

We investigate the performance of these detection statistics by comparing a plot of its EDD versus $\log(\mbox{ARL})$. From theoretical analysis, we expect such plot of the CUSUM statistic to be close to linear. We first introduce a simple baseline: the Shewhart control chart \citep{shew-jamstaa-1925,shew-book-1931}, which counts the number of events in a sliding window and stops when the number of events falls out of a specific range:
$$
\Tsh = \inf \Bigg\{t:\, \sum_{i\in [D]} (N_t^i - N_{t-w}^i) \notin [b_1,b_2] \Bigg\}.
$$
This Shewhart chart can detect changes in average intensity. Note that it does not take into account the network structure or the location of events. In this example, we only consider the case where the average intensity will be increased after the change-point, and thus choose a one-sided interval by letting $b_1 = 0$.

Figure~\ref{fig:ARL_EDD} shows the comparison results. For both networks, the CUSUM procedure with exact post-change parameters achieves the best performance, followed by the score statistic and the GLR. All three methods are better than the baseline Shewhart chart. Window lengths are chosen numerically for the latter three procedures, such that the performance is (approximately) optimal for an ARL between 500 and 50000. The details are shown in Table \ref{tbl:1}. When simulating the EDD, we set the change-point $\kappa$ to be slightly larger than the window length $w$ so that it is possible for the EDD to be less than $w$.

\begin{figure}[htb!]
\centering
\begin{tabular}{cc}
\includegraphics[width=0.47\textwidth]{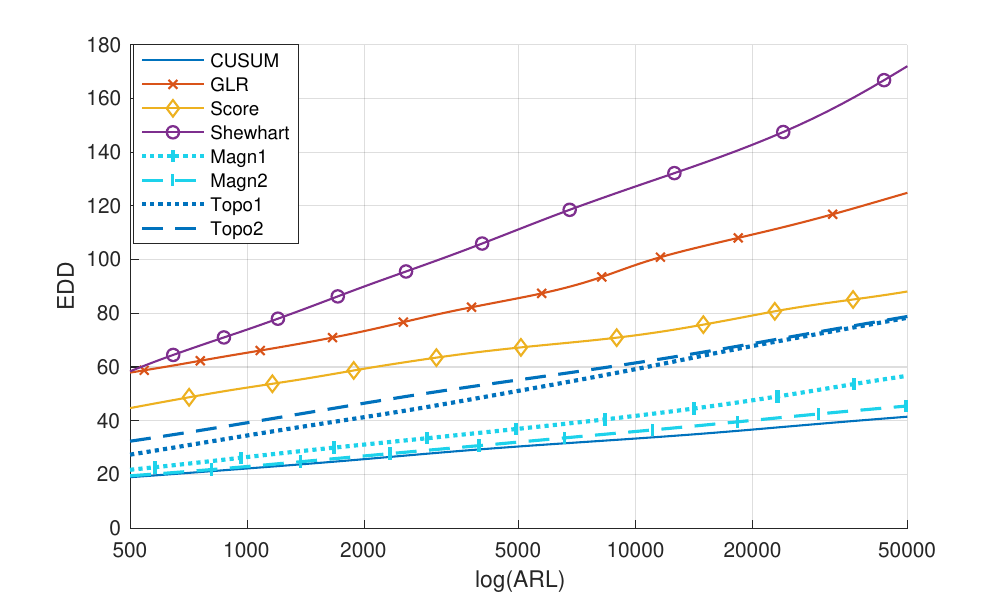}&
\includegraphics[width=0.47\textwidth]{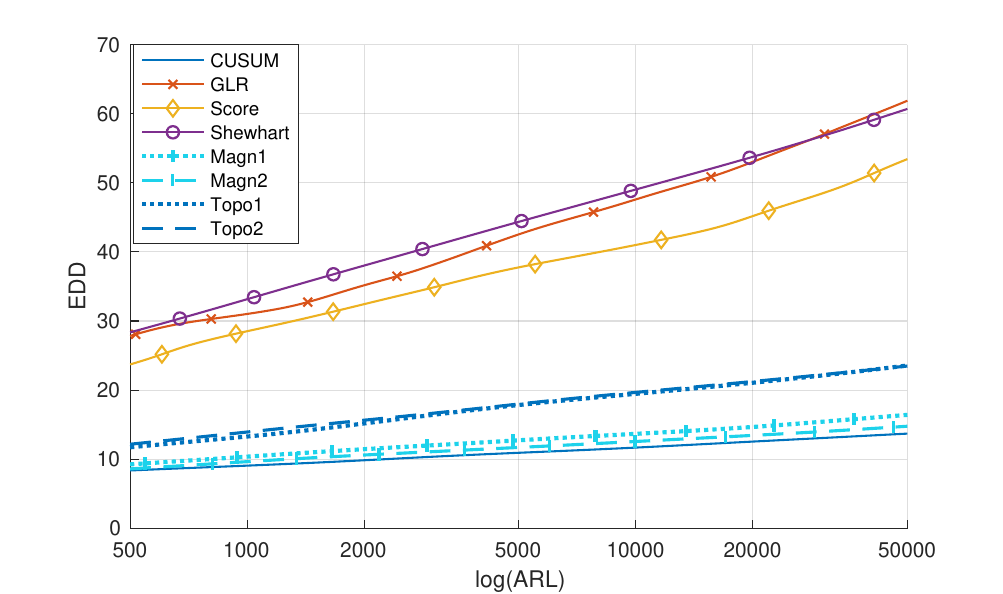}\\
(a)&(b)
\end{tabular}
\caption{Performance comparison of different detection statistics on (a) network of size 8, (b) network of size 100.}
\label{fig:ARL_EDD}
\end{figure}

\begin{table}[htb!]
    \centering
    \begin{tabular}{c|c|c|c|c|c}
    \hline
    & & CUSUM & GLR & Score & Shewhart\\\hline
    \multirow{2}{*}{Network size = 8} & $w$ & NA & 80 & 80 & 120\\\cline{2-6}
    & time &1.508& 37.14& 3.013& 0.046\\\hline
    \multirow{2}{*}{Network size = 100} & $w$ & NA & 40 &40 &60 \\\cline{2-6}
    & time &1.304 &54.89 &12.71& 0.030\\\hline
    \end{tabular}
    \caption{Optimal window length and running time of selected procedures.}
    \label{tbl:1}
\end{table}

\textbf{Running time.} Table \ref{tbl:1} also summarizes the average time in seconds needed to compute the selected procedures over a time horizon 50000, with no change-point on a personal computer (Apple M1 chip). The base intensity of the networks are scaled such that the two networks have similar average number of events. Clearly apart from the Shewhart chart, CUSUM enjoys the quickest running time and remains computationally efficient as the network size increases, thus demonstrating the scalability of the proposed approach.

\textbf{Misspecification in post-change parameters.} We also consider the CUSUM procedure when the assumed post-change parameter differs from the true parameters in Figure \ref{fig:ARL_EDD}, indicating a model mismatch. For the small network of size 8, while the real change occurs on two pairs of nodes making the influence factor both 0.4, we consider a CUSUM procedure which assumes a correct post-change network topology, but with magnitudes of the post-change parameters on the edges to be 200\% (Magn1) and 50\% (Magn2) of the true magnitudes, respectively. We also consider a CUSUM procedure which assumes the correct magnitude of the post-change parameter, but an incorrect post-change network topology. In one case, a change in the influence from $2\to 1$, $1\to 3$, $6\to 1$, $1\to 7$ is expected (Topo1), and in another, only the change in the influence from $2\to 1$ is expected (Topo2). 
 For the large network of size 100, we consider the case when the magnitudes of the post-change parameters are 200\% (Magn1) and 50\% (Magn2) of the true ones. For the topological model mismatch, recall that the true change happens on 200 edges, we select 200 more edges (Topo1) or drop 50 edges (Topo2), both at random as the misspecified cases. Even with misspecified post-change parameters in either influence magnitude or network topology, the CUSUM procedure can still achieves better performance than the GLR and the score procedures for both network settings. This demonstrates that the proposed CUSUM procedure is reasonably robust to the misspecification of the post-change parameters.

Though misspecified, the cases provided above still partly capture the true change. When the estimated post-change parameters deviate greatly from the true parameters, the CUSUM procedure can fail to achieve an EDD linear in $\log(\text{ARL})$, as shown in Figure \ref{fig:misspec}. For the network of size 8, the post-change parameters is estimated to perceive a change in the influence from $6\to 1$, $1\to 7$, with a magnitude of $0.4$. For the network of size 100, the misspecified post-change parameters select 200 edges randomly (independent of the true topology of the change) with a magnitude of $0.2$.

{\color{black} In real scenarios, the abrupt change may represent an unexpected anomaly, and we do not have enough data to estimate the post-change parameters. In such cases, we may choose a targeted topology of the post-change parameters to detect a certain type of structural change and choose the magnitude to reflect a minimum size of the change to be detected. For certain applications, it is also possible to enumerate the potential changes and run several detection procedures in parallel, each responsible for monitoring the process against one type of change. We can also see which type of change causes an alarm to help identify the change pattern and location. Alternatively, there are also adaptive CUSUM procedures \citep{xie2020sequential}, which use ``future'' samples to estimate a potential post-change parameter and use as a plug-in estimator in the CUSUM statistics. However, such methods may incur an additional delay in detection. }

\begin{figure}[htbb!]
\centering
\begin{tabular}{ccc}
\includegraphics[width=0.30\textwidth]{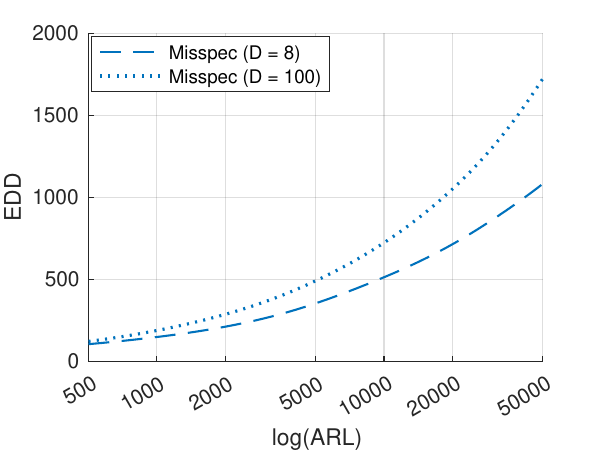}&
\includegraphics[width=0.30\textwidth]{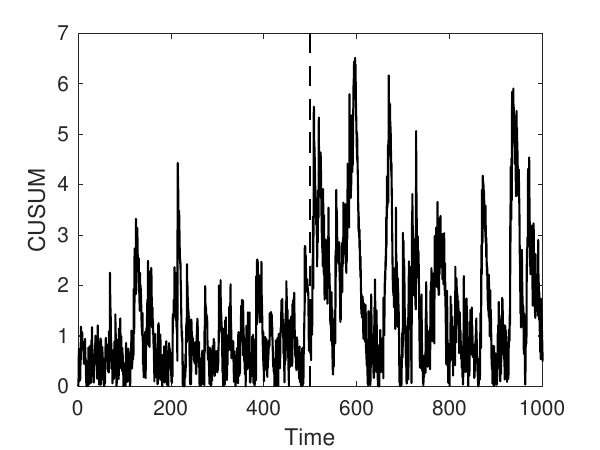}&
\includegraphics[width=0.30\textwidth]{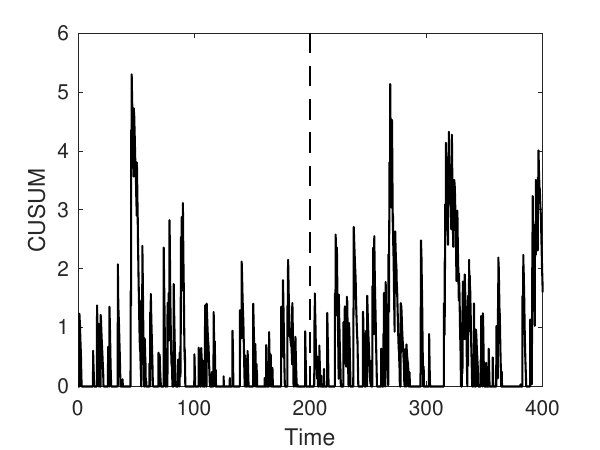}\\
(a)&(b)&(c)
\end{tabular}
\caption{Examples when CUSUM fails to detect under misspecified post-change parameters. (a) The EDD vs. $\log(\text{ARL})$ plot. (b)(c) Examples of the CUSUM procedure under misspecification for the network of size 8 and 100, using the same sequence of events with Figure \ref{fig:stats}, respectively.}
\label{fig:misspec}
\end{figure}

\subsection{Effect of Grid Size $\gamma$} 

As mentioned earlier, the choice of the grid size $\gamma$ involves a trade-off between algorithm performance and computational complexity. To investigate this trade-off, we compare the proposed CUSUM, the GLR, and the score statistics with a grid size $\gamma$ ranging from 0.1 to 50 on the network of size 8. For the EDD, we assume that the change-point is uniformly distributed between two grid points. Figure~\ref{fig:gamma} shows the effect of the grid size on CUSUM, the GLR, and the score procedure. We see that a large grid size $\gamma$ results in both a larger ARL and EDD. If we instead tune the threshold $b$ to fix the ARL, the EDD may still increase with a larger $\gamma$. 

To understand the effect of $\gamma$ on computation complexity, we will consider the GLR statistic as the computation for the GLR is the most expensive. To solve the convex optimization problem for each window, we use the EM algorithm as described in \cite{li2017detecting}, and terminate when the update in the log-likelihood is less than $10^{-4}$. Figure~\ref{fig:gamma}(c) shows the average iterations needed per window for different grid sizes. We see that, as $\gamma$ increases, the computation required (in terms of number of iterations) increases as well, which matches intuition.

\begin{figure}[htbp!]
\centering
\begin{tabular}{ccc}
\includegraphics[width=0.29\textwidth]{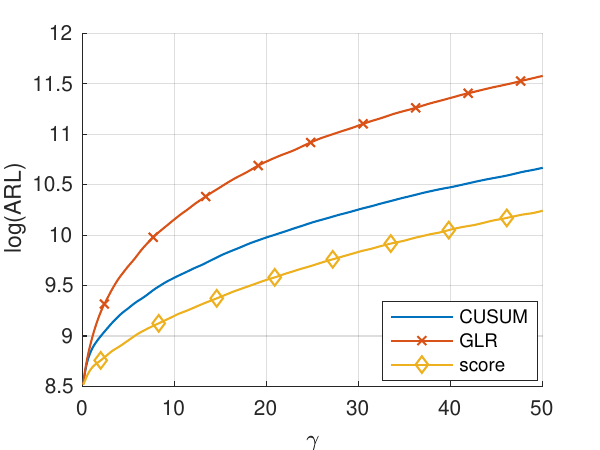} & \includegraphics[width=0.29\textwidth]{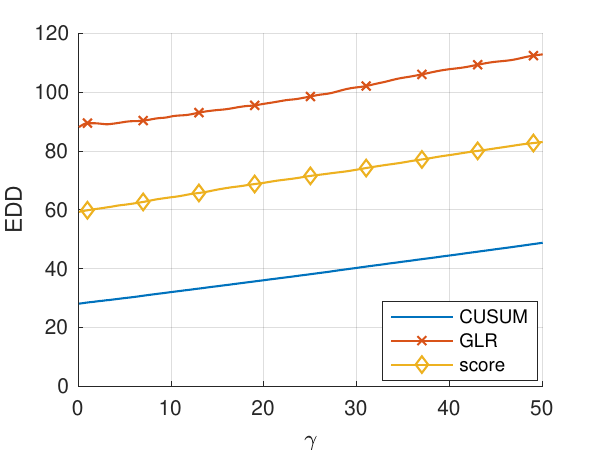} & \includegraphics[width=0.29\textwidth]{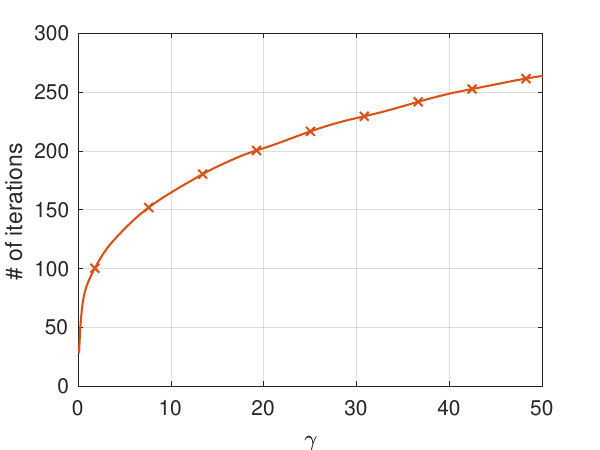}\\
(a) ARL for fixed threshold $b$ & (b) EDD for ARL $=5000$ & 
(c) Iterations per window 
\end{tabular}
\caption{Effect of the grid size $\gamma$ on CUSUM, GLR, and score statistics. The window length for the GLR and the score statistics is 60 and for Shewhart is 120, which are optimized for better EDD performance. (a) shows the effect of $\gamma$ on the ARL at fixed threshold $b = 6.319, 37.66, 148.2$ for CUSUM, the GLR and the score statistics respectively, making ARL $=5000$ at $\gamma = 0.1$. (b) shows the effect of $\gamma$ on the performance by tuning $b$ so that ARL is 5000. (c) shows the average number of iterations per window needed in the GLR statistic for the optimization problem to converge.}
\label{fig:gamma}
\end{figure}

\section{Detecting Neuronal Network Population Code Change}\label{sec:realdata}

We now return to the motivating problem of detecting population code changes in neuronal networks. The data considered are neural spike trains, which record the sequence of times when a neuron fires an action potential. The multivariate Hawkes processes from Section~\ref{sec:multivariate_Hawkes} have been used for modeling spike train data \citep{lambert2018reconstructing,wang2020uncertainty}, and capture two appealing features for neuronal networks. First, the base intensities $\mu_i$ capture noisy influences on neuron $i$, resulting from either unobserved neurons or external stimuli. Second, the influence parameters $\alpha_{ij}$ capture the functional influence from neuron $j$ to neuron $i$ due to electrochemical dynamics. 

We are interested in detecting the change-point in the underlying population code from neural data. These are {\it abrupt} changes, as the behavior of populations of neurons respond quickly (usually in just a few ms) to changing input. Population codes are a distributed representation of information used widely across many neural architectures and have been most widely documented in the cortex. As opposed to dense representations, population codes consist of sparsely activated subsets of neurons in which the information is distributed amongst the entire subset. Figure~\ref{fig:popcode} illustrates this idea. The left plots show a plausible neuronal network topology for the population coding of seeing a cat or a dog. The right plots show the corresponding spike train data on the neuronal network, as one changes states from a cat to a dog. Identifying this population code change-point from experimental data provides scientists a better understanding of the role of each neuron, which can be used for repairing neuronal networks.
\begin{figure}[htbp!]
\centering
\begin{tabular}{cc}
\includegraphics[width=0.41\textwidth]{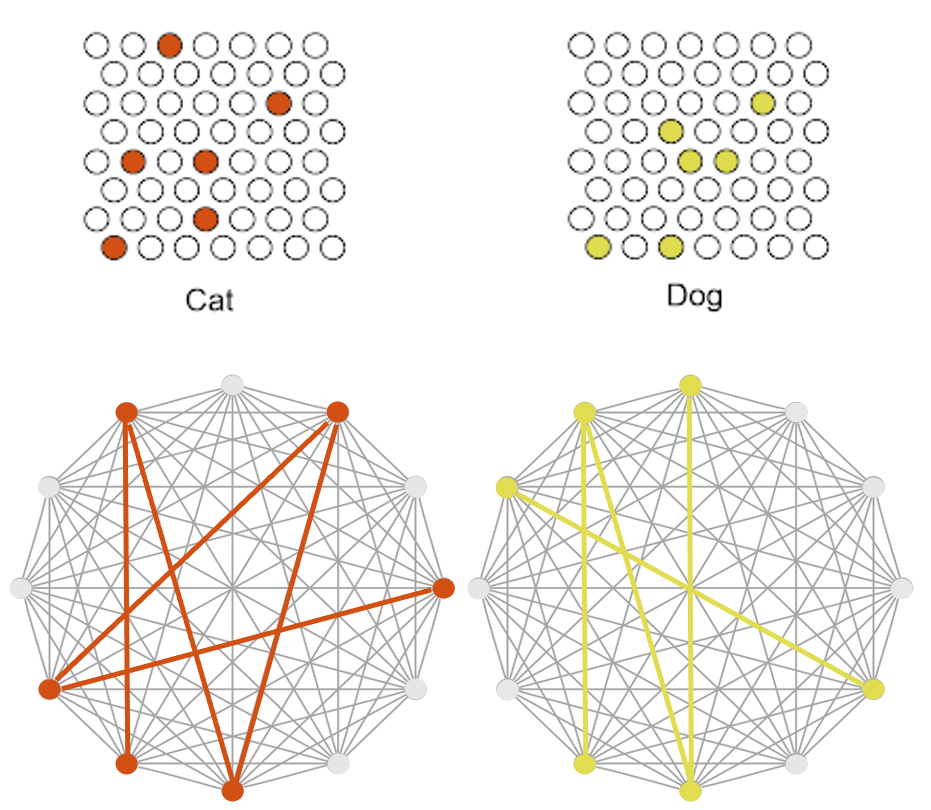}&
\includegraphics[scale=0.55]{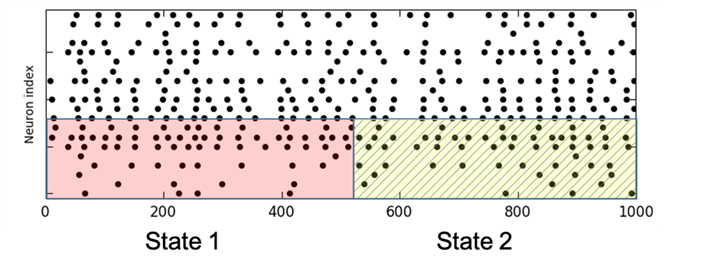}\\
\end{tabular}
\caption{(Left) A plausible network topology for `Cat' and `Dog'; (Right) Spike train data for each neuron under a population code change.}
\label{fig:popcode}
\end{figure}

Here, the \textit{sequential} nature of change-point detection can be advantageous for practical implementation. Real-time detection of biological neural networks is known as “continuous detection” in the neuroscience literature (see, e.g., \citealp{goense2003continuous}). This is in contrast to the more standard ``trial-based'' (or fixed-sample) testing, where the beginning and duration of the testing interval are pre-determined. A key advantage of continuous detection over trial-based testing is a reduction in experimental sample size \citep{goense2003continuous}: the experiment terminates after a change-point is detected, and does not need to run for the full testing period. This yields considerable cost savings for experiments, and speeds up the decision-making procedure. Continuous change-point detection is a capability vital to the success of neural interfaces, devices that monitor and decode the activity of a subject's brain. These neural interfaces being widely used in modern neuroengineering problems to restore capabilities to patients, e.g., manipulating a robotic hand or even typing on a virtual keyboard. Underlying the use of real-time neural interfaces is the ability to detect changes in population codes in real time.

While there have been significant advances in neuroimaging technology, it can still be quite costly to record fine-scale spiking data through in-vivo (i.e., \textit{physical}) experiments. To illustrate the proposed method, we instead simulate the spike train data using the PyNN package with the NEURON simulator, which implements the neuronal model in \cite{brette2005adaptive}. We build off previous work in neural simulation and use a network of exponential integrate-and-fire neurons with spike triggered and sub-threshold adaptation currents. This can be viewed as a \textit{computer} experiment surrogate for the expensive physical experiments, which we cannot obtain due to high costs.

The simulation set-up is as follows. We first simulate several small networks of neurons in a balance of 80-20 excitatory to inhibitory neurons, with network size fixed at $D$ = 14 neurons. From this, we obtain a continuous readout of each neuron spiking data. Each neuron then receives a small Gaussian noise current, representing random external influence on the network. In addition, a select few neurons receive inputs from an external source, which represents the phenomenon of sparse population coding. The neurons that spike at higher rates form a distributed representation of the network state.

We then randomly selected two such subsets of neurons, representing two different states. We simulate the network in the first state for a long time (from $t$ = 0 - 20,000 ms) to learn network dynamics and structure under the first population code. The pre-change Hawkes process parameters are obtained via maximum likelihood estimation (MLE) on the pre-change spike train data. After $\kappa$ = 20,000 ms, we then simulate a change from the first to the second state. The goal is to quickly detect the change-point $\kappa$ in a sequential fashion from the simulated spike trains. For CUSUM, the post-change Hawkes process parameters are estimated via MLE on the post-change spike trains. For the score statistic, the pre-change Fisher information matrix $I_0$ is highly ill-conditioned when estimated from spike trains, so we instead use a slightly regularized estimate $I_0 + \lambda I$, where $\lambda = 1$ and $I$ is an identity matrix. The estimated pre-change and post-change models are shown in Figure~\ref{fig:neuron_model}.

\begin{figure}[htbp!]
\centering
\begin{tabular}{ccc}
\includegraphics[width = 0.3\linewidth]{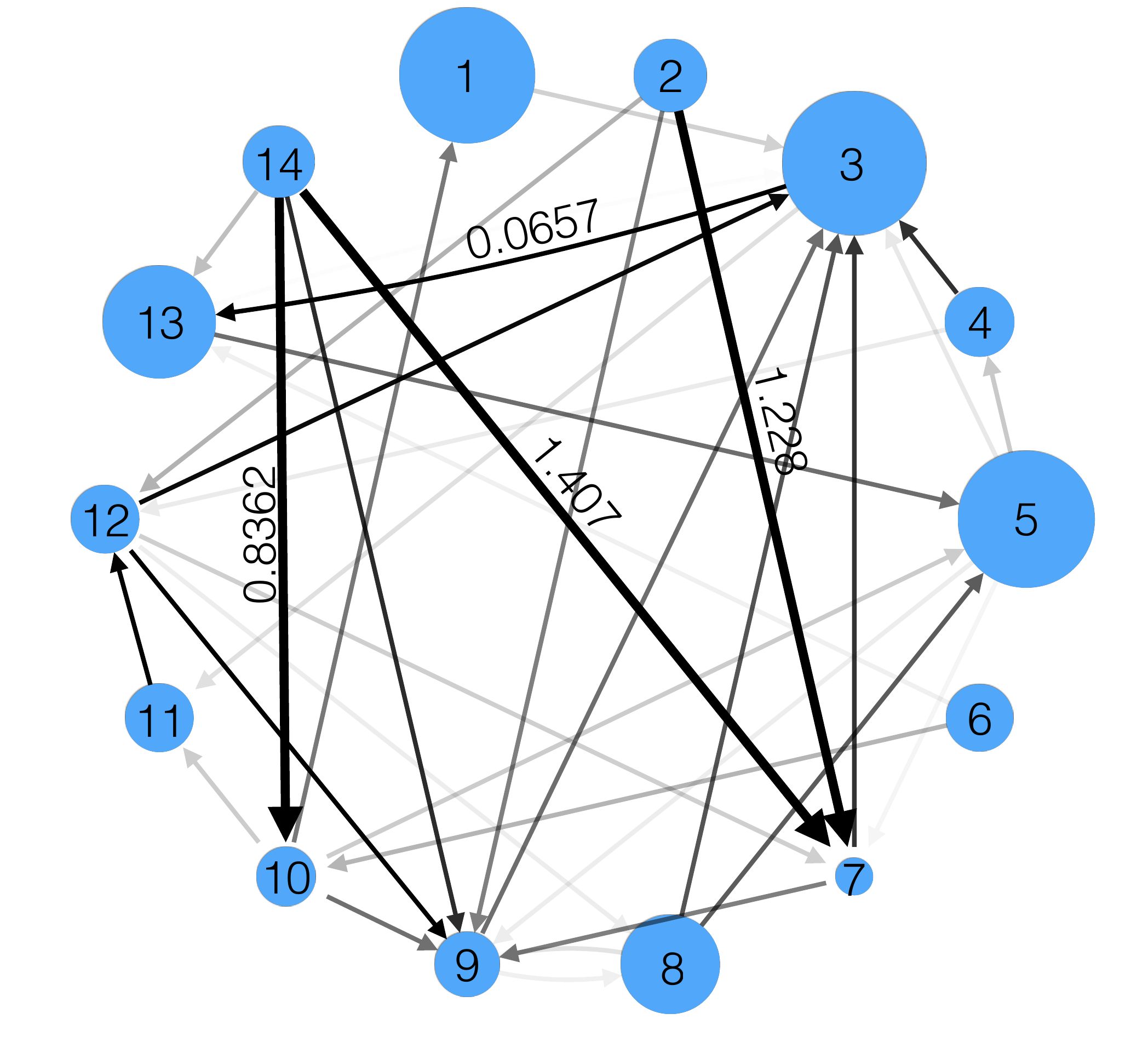} & \includegraphics[width = 0.3\linewidth]{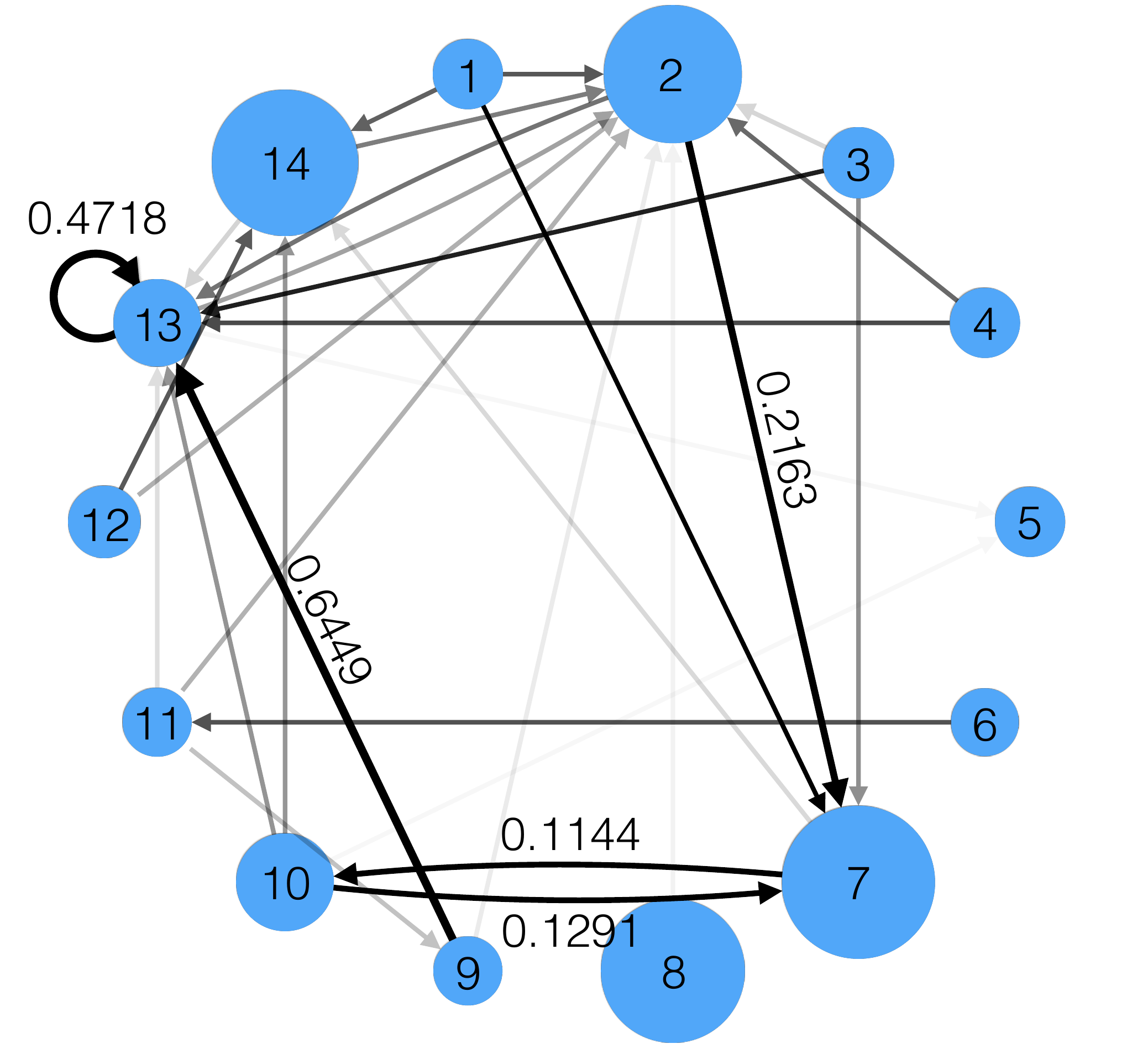} & \includegraphics[width = 0.3\linewidth]{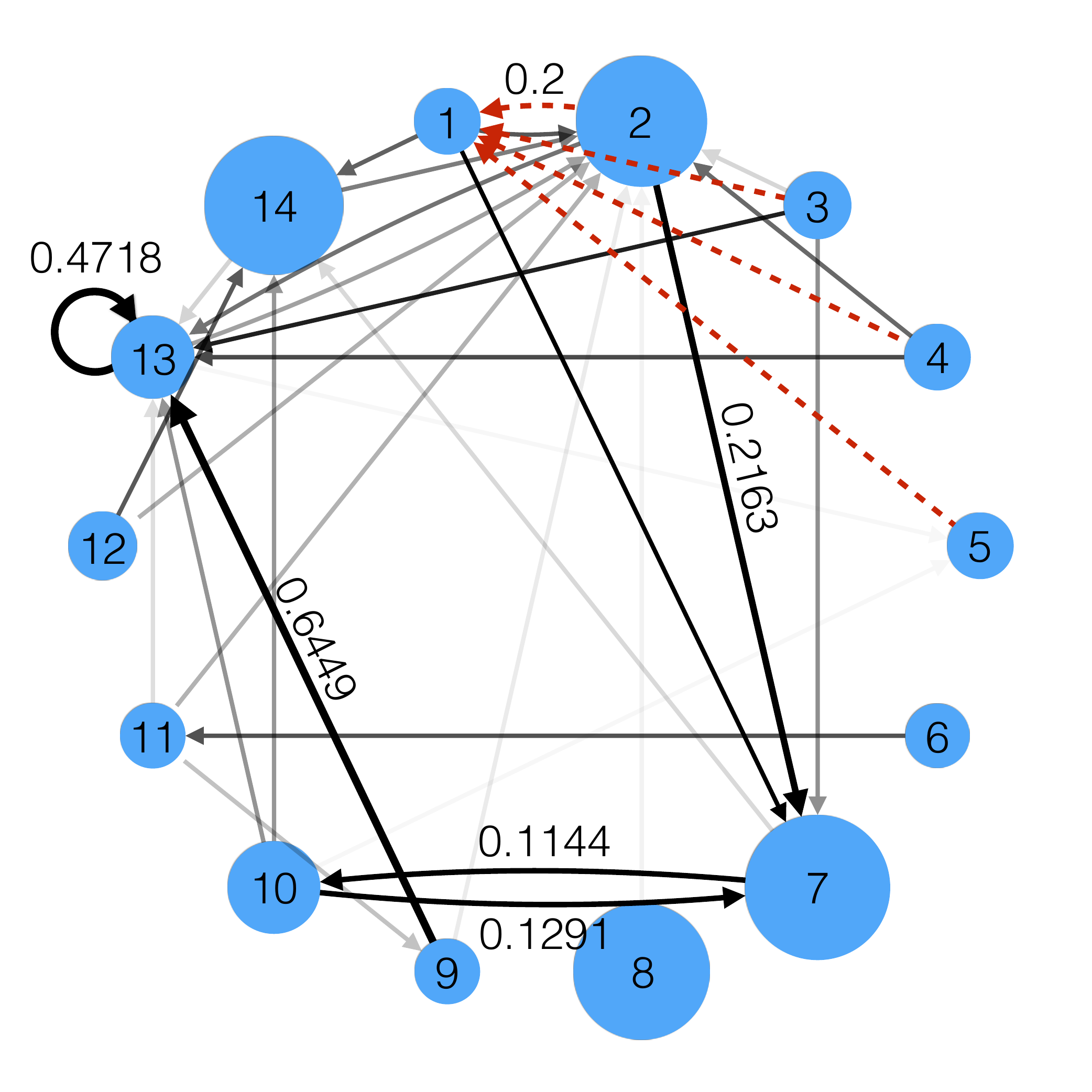}\\
(a) Pre-change & (b) Post-change & (c) Mis-specified post-change
\end{tabular}
\caption{(a)(b) shows the estimated pre-change and post-change parameters. The size of the nodes are proportional to their background intensities ranging from 0 to 0.016. The opacity of the unlabeled edges are proportional to their weights, ranging from 0 to 0.06. (c) shows a case of post-change topology misspecification in CUSUM statistic.}
\label{fig:neuron_model}
\end{figure}

\begin{figure}[htb!]
\centering
\begin{tabular}{ccc}
\includegraphics[width=0.3\textwidth]{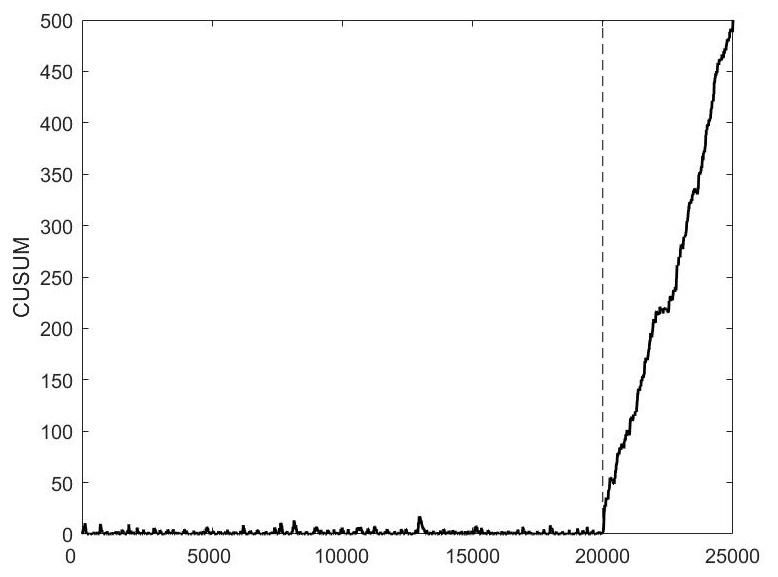} & \includegraphics[width=0.3\textwidth]{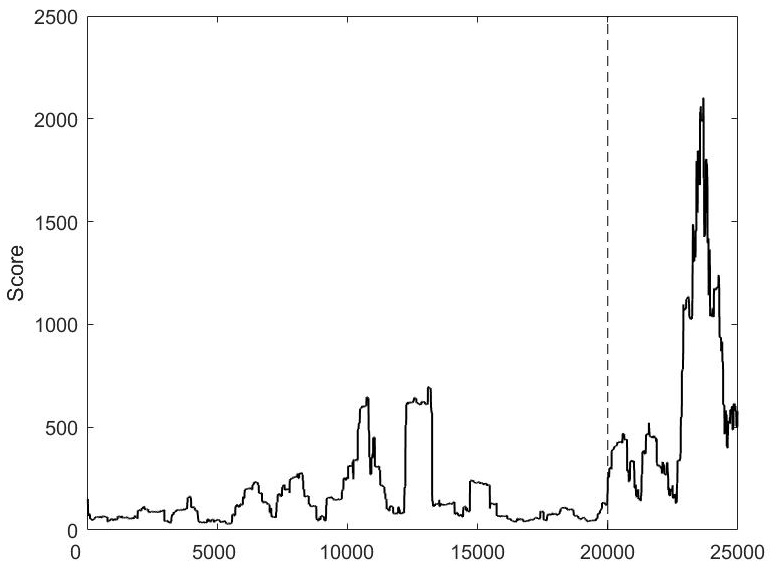} & \includegraphics[width=0.3\textwidth]{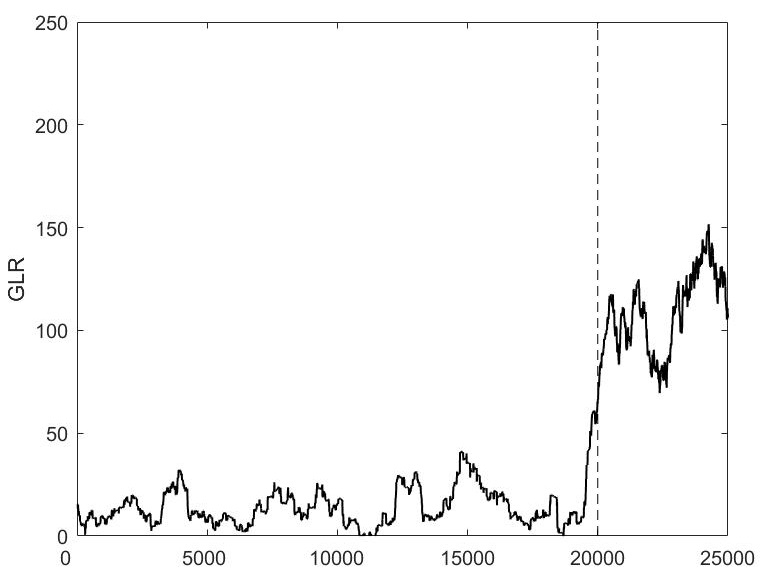}\\
(a) CUSUM & (b) Score stat & (c) GLR
\end{tabular}
\caption{The CUSUM, score, and GLR statistics for the neuronal network application with $D=14$ nodes. The true change-point is indicated by the dashed line.}
\label{fig:neurostats}
\end{figure}

\begin{figure}[htbp!]
\centering
\begin{tabular}{ccc}
\includegraphics[width=0.3\textwidth]{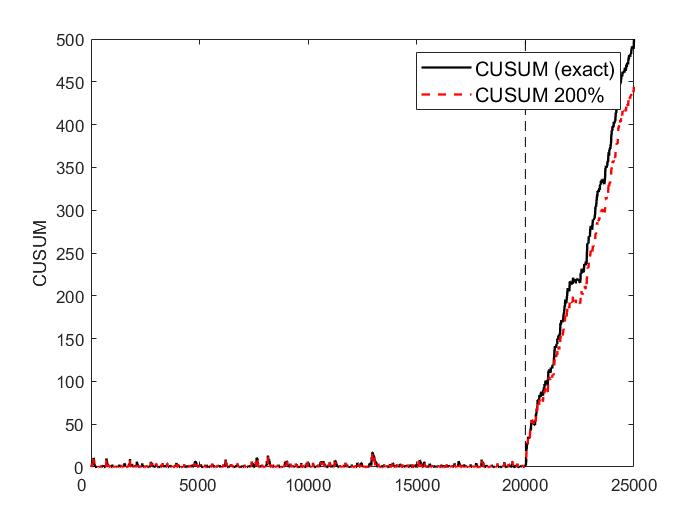} & \includegraphics[width=0.3\textwidth]{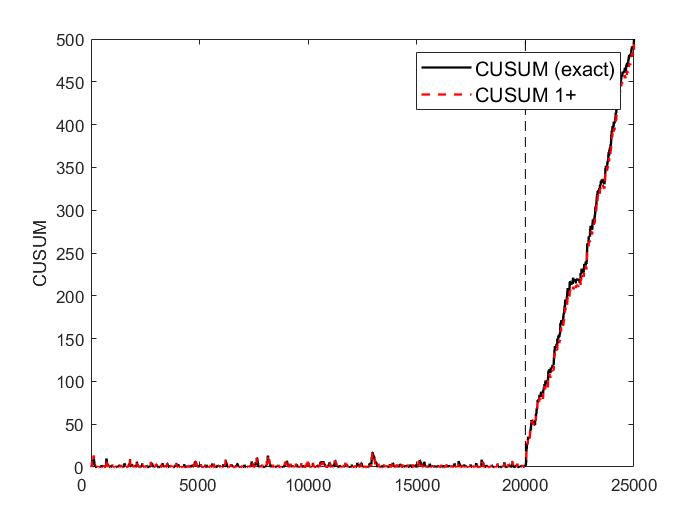} & \includegraphics[width=0.3\textwidth]{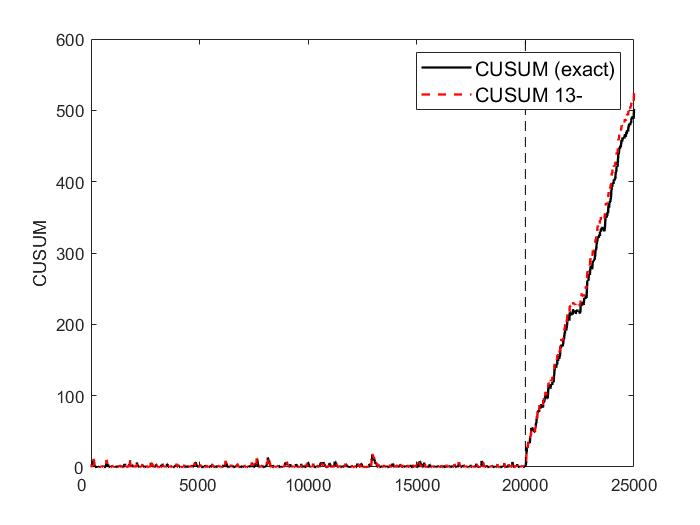}\\
(a) Influence misspecification & (b) Added edges & (c) Missing edges 
\end{tabular}
\caption{CUSUM statistics under different misspecifications of post-change parameters. The red line shows the CUSUM statistic under an ``exact'' specification of post-change parameters, and the dash line indicates the change-point.}
\label{fig:neurostatsms}
\end{figure}

Figure~\ref{fig:neurostats} shows the CUSUM, GLR, and score statistics, respectively, with the dashed line indicating the change-point in population code. Both the score and GLR statistics utilize a window size of 1,000 and an update rate of $\gamma$ = 5. As in numerical experiments, we see that the CUSUM statistic increases rapidly after the change-point, which shows it is quite effective at detecting the underlying neuronal network changes. The score and GLR statistics are also noticeably larger after the change-point, with the increase in GLR more prominent than the increase for the score statistic. The increases in GLR and score statistics are noticeably lower than that for the proposed CUSUM procedure, which suggests that our method can better detect population code changes in neuronal networks.


Next, we consider the case where post-change parameter estimates are misspecified for the CUSUM statistic. This may arise, e.g., when there is a lack of spike train data for post-change parameter estimation. We consider three scenarios for misspecification: (a) the post-change topology is correct, but the influence parameters are scaled at 200\%, (b) the influence parameters are correct, but there are spurious edges on neuron 1 for the topology (see Figure \ref{fig:neurostatsms}(b)), (c) the post-change influence parameters are correct, but all the edges to neuron 13 are missing for the topology. Figure~\ref{fig:neurostatsms} shows the CUSUM statistics for these three scenarios, along with the ``exact'' CUSUM statistics, which use exact post-change MLEs. We see that the CUSUM is quite robust: its CUSUM statistics are quite close to the exact CUSUM for both influence and topology misspecifications. Hence, our method appears to efficiently detect population code changes, even under uncertainties in post-change parameter estimation.

\begin{figure}[htbp!]
\centering
\begin{tabular}{ccc}
\includegraphics[width=0.3\textwidth]{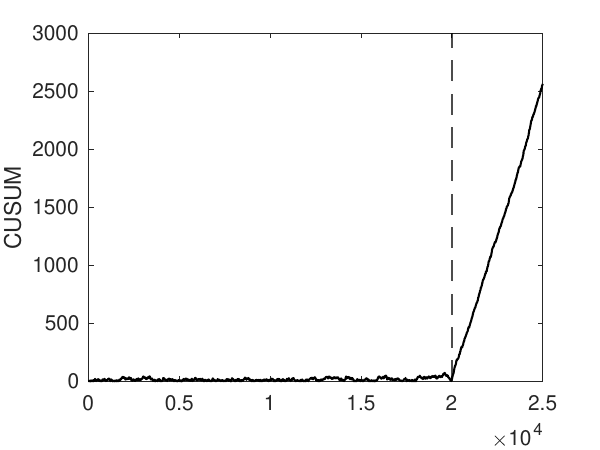} & \includegraphics[width=0.3\textwidth]{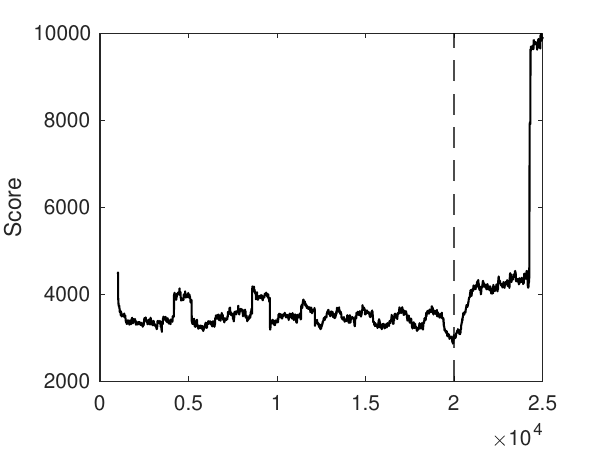} & \includegraphics[width=0.3\textwidth]{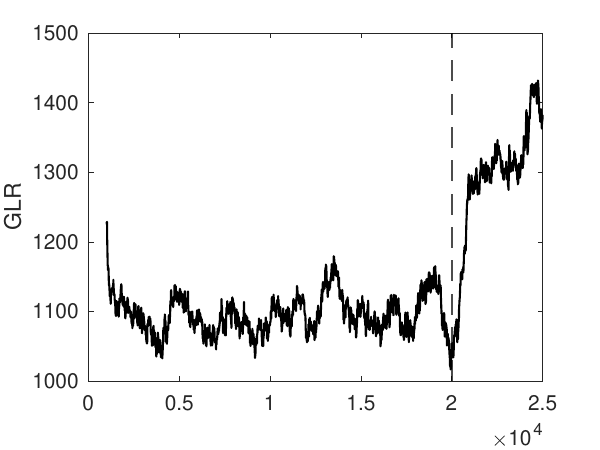}\\
(a) CUSUM & (b) Score stat & (c) GLR
\end{tabular}
\caption{The CUSUM, score, and GLR statistics for the neuronal network application with $D=80$ nodes. The true change-point is indicated by the dashed line.}
\label{fig:neurostats2}
\end{figure}

Finally, we investigate the scalability of these methods by increasing the network size to $D = 80$ neurons, with all simulation and experimental settings fixed as before. {\color{black} Due to the inaccuracies in estimating the pre-change parameters, the CUSUM procedure has a positive drift in the pre-change scenario, which disagrees with the theoretical property needed for CUSUM -- before the change, the expected drift should be negative; otherwise, there will be constant false alarms raised. To address this issue, as a common practice, we subtract a positive constant 0.2 from the increment when forming the CUSUM procedure from all $\ell_{n\gamma,\cdot}$, which will ensure the drift term has a negative expected value.} Figure \ref{fig:neurostats2} shows the corresponding CUSUM, GLR, and score statistics, with the dashed line indicating the change-point in population code. We see similar observations as before: the CUSUM statistic increases rapidly after the change-point, whereas the increases in the GLR and score statistics are much more subtle, thus indicating our approach can provide better detection of population codes. As in simulation experiments, the computation time favors our method: the proposed CUSUM procedure requires 0.44 seconds on the 80-node network, whereas the GLR and score procedures require 32.48 and 4.63 seconds, respectively. This shows the improved performance and efficiency of our recursive CUSUM approach.

\section{Conclusion and Discussions}\label{sec:conclusion}
We have presented a new sequential CUSUM procedure for detecting change-point in the multi-dimensional self- and mutual-exciting point processes, i.e., network Hawkes processes. By tackling the complex and long-term dependence between event times, we develop the CUSUM procedure that enjoys efficient recursive computation and memory efficiency if we employ truncation. Using numerical experiments, we showed that the CUSUM procedure yields improved performance over existing detection procedures (Shewhart-type) based on score statistics and generalized likelihood ratio (GLR) statistics. Moreover, we found that, although the CUSUM procedure requires specifying the post-change distribution parameters, it is fairly robust to parameter misspecification {\color{black} when it is possible to estimate the topology and magnitude of a potential abrupt change} and outperforms existing methods in that setting. This can be partly explained by that these alternative methods are the Shewhart-type approaches (based on evaluating a detection statistic using a sliding window), which does not accumulate information from the past. 
We also demonstrated a realistic neuroengineering application of our procedure for neuronal network change-point detection.

\section*{Acknowledgement}

The work of Haoyun Wang, Liyan Xie, and Yao Xie were partially supported by an NSF CAREER  CCF-1650913, and NSF DMS-1830210.



\nocite{tartakovsky2019sequential}
\nocite{reynaud2007some}
\bibliographystyle{apalike}
\bibliography{Hawkes}


\newpage
\begin{center}
{\Large \bf Supplementary Material for ``Sequential Change-Point Detection for Mutually Exciting Point Processes''}
\author{\large Haoyun Wang, Liyan~Xie, Yao~Xie\hspace{.2cm}\\
    H. Milton Stewart School of Industrial and Systems Engineering\\
     Georgia Institute of Technology\\
    and \\
    Alex Cuozzo, Simon Mak\hspace{.2cm} \\
    Department of Statistical Science, Duke University}
\end{center}

\appendix

\section{Proof of Lemma 1}
\begin{proof}
For fixed event and any $\tau \in (t_k,\min\{t_{k+1},t\}]$, the intensity $\lambda_{i,\tau}(s)$ at any time $s\geq\tau$ is a constant that does not depend on $\tau$, and the first part of the log-likelihood ratio \eqref{eq:log-likelihood_ratio},
$$
\sum_{i\in[D]}\int_{\tau}^t\log \left( \frac{\lambda_{i,\tau}(s)}{\lambda_{i,\infty}(s)} \right) dN_s^i,
$$
is also a constant. 
For the second part, 
$$
-\sum_{i\in[D]}\int_{\tau}^t (\lambda_{i,\tau}(s)-\lambda_{i,\infty}(s)) ds,
$$
the integrand is no larger than 0 because $\lambda_{i,\tau}(s) = \mu_i(s) \leq \lambda_{i,\infty(s)}$ for any $\tau<s\leq t_{k+1}$. Therefore the supremum of $\ell_{t,\tau}$ is reached when $\tau \to t_k^+$, which is $$
\sup_{t_k<\tau\leq \min\{t_{k+1},t\}}\ell_{t,\tau} = \ell_{t,t_k^+}.
$$
For the supremum w.r.t. $\tau$ over $[0,t_1]$, the equality can be derived using similar arguments.
\end{proof}

\section{Theoretical properties}
\label{sec:theoretical_properties}

We study the theoretical properties of the exact CUSUM procedure, as presented in Algorithm~\ref{alg:CUSUM}. The theoretical analysis is important for a practical implementation of the proposed method, since it guides a proper choice of threshold for the detection algorithm. 

We first establish a lower bound for ARL regarding the detection threshold $b$. Typically for the i.i.d. observation setting, the ARL grows exponentially with respect to the threshold $b$. Although the considered Hawkes process model is a continuous time procedure, it is naturally discretized by event times. Interestingly, even in this continuous setting, we obtain a similar result for the ARL of the proposed procedure:
\begin{theorem}[ARL of CUSUM]
In Algorithm~\ref{alg:CUSUM} under $H_0$, the number of events happened before $T_{\scriptscriptstyle \text{\emph{C}}}$ satisfies
$
\mathbb E_\infty\left[|\{k:t_k\leq T_{\scriptscriptstyle \text{\emph{C}}}\}|\right]\geq
\mathbb E_\infty\left[|\{k:t_k\leq\hat\tau\}|\right]\geq \left(\frac{1-o(1)}{2}\right)e^b.
$
\label{thm:arl}
\end{theorem}
%

We investigate next the EDD of the proposed algorithm. For i.i.d. observations, the EDD of CUSUM procedures is on the order of $\log(\mbox{ARL})$ divided by the Kullback-Leibler (KL) divergence for the pre- and post-change distributions. Similar results can be obtained for CUSUM with non-i.i.d. observations \citep{tartakovsky2014sequential,tartakovsky2019sequential}. We expect a similar result may hold for the CUSUM procedure for Hawkes process, although the complete proof is complicated, which we leave for future research. 

\begin{remark}[EDD of CUSUM]
Let 
$I_{\scriptscriptstyle \text{\emph {KL}}} = \lim_{t\to\infty}t^{-1}\ell_{t,0}$
be the asymptotic KL divergence between the post-change and pre-change processes (since the conditional intensity function for Hawkes process is stochastic). For any change-point $\kappa\geq 0$ and any event data $\mathcal H_\kappa$ up to $\kappa$, we expect the EDD to be
$\mathbb E_\kappa\left[T_{\scriptscriptstyle \text{\emph{C}}}-\kappa|\mathcal H_\kappa\right] \leq (b/I_{\scriptscriptstyle \text{\emph {KL}}})(1+o(1)).$
\label{thm:edd}
\end{remark}

The KL divergence between different Hawkes models using mean-field approximation is summarized in \cite{li2017detecting}. In particular, the KL divergence for the model shown in Equation \eqref{hypothesis test} is given by
$
\IKL = (\bar{\boldsymbol\lambda}_1)^T(\log\bar{\boldsymbol\lambda}_1 - \log\bar{\boldsymbol\lambda}_0) - \boldsymbol 1^T (\bar{\boldsymbol\lambda}_1 - \bar{\boldsymbol\lambda}_0),
$
where $\bar{\boldsymbol\lambda}_1 = (I_D-A_1)^{-1}\boldsymbol\mu$ is the expected intensity for post-change distributions, and $\bar{\boldsymbol\lambda}_0 = (I_D-A_0)^{-1}\boldsymbol\mu$ is the expected intensity for pre-change distribution. Here $I_D$ is the identity matrix, $\boldsymbol 1$ is a vector of ones, $\boldsymbol\mu = (\mu_i)_{i\in[D]}$ is the constant base intensity vector, and $A_k = (\alpha_{ij,k})_{i,j\in[D]}, k=0,1$.  
Quantifying the KL-divergence between the pre- and post-change distributions can help us to understand whether a case is easy or difficult to detect.

\section{Proof of Theorem \ref{thm:arl}}
\begin{proof}
For each fixed $\tau$, $(\exp \ell_{t,\tau})_{t\geq \tau}$ is a martingale w.r.t. $t$, with $\exp\ell_{\tau,\tau} = 1.$ By Ville's maximal inequality for non-negative supermartingales, we have 
\[
\mathbb P_\infty[\exists t\geq \tau, \exp\ell_{t,\tau}\geq e^b]\leq e^{-b},
\]
and for each $k$,
\[
\mathbb P_\infty[\hat\tau = t_k^+]\leq e^{-b}.
\] 
By union bound, for any $k$,
\[
\mathbb P_\infty[\hat\tau\geq t_k]=\mathbb P_\infty[\hat\tau\notin \{t_i^+, i< k\}\cup \{0\}]\geq 1-\mathbb P_\infty[\hat\tau = 0] - \sum_{i< k}\mathbb P_\infty[\hat\tau = t_i^+] \geq 1-ke^{-b}.
\]
\[
\mathbb E_\infty\left[|\{k:t_k\leq \hat\tau\}|\right] = \sum_{k=1}^\infty \mathbb P_\infty[\hat\tau\geq t_k]\geq \sum_{k=1}^{\lfloor e^b\rfloor}(1-ke^{-b}) \geq \left(\frac{1-o(1)}{2}\right)e^b.
\]
And by the definition of $\hat\tau$, there is always $\Tc>\hat\tau$. Thereby we complete the proof.
\end{proof}

\section{Discussions on Remark \ref{thm:edd}}
Here note that, $I_{\scriptscriptstyle \text{\emph {KL}}}$ is irrelevant with $\kappa$ or $\mathcal H_{\kappa}$. We show the result on the EDD for the one-dimensional Hawkes process, where the kernel function $\varphi$ has finite support and is upper bounded.
\begin{align*}
\mathbb E_\kappa\left[T_{\scriptscriptstyle \text{\emph{C}}}-\kappa|\mathcal H_\kappa\right] \leq&\  \frac{b}{I_{\scriptscriptstyle \text{\emph {KL}}}} + \int_{b/I_{\scriptscriptstyle \text{\emph {KL}}}}^\infty \mathbb P(T_{\scriptscriptstyle\text{\emph{C}}}-\kappa\geq t)dt\\
=&\ \frac{b}{I_{\scriptscriptstyle \text{\emph {KL}}}}\left(1 + \int_{1}^\infty \mathbb P\left(T_{\scriptscriptstyle\text{\emph{C}}}-\kappa\geq \frac{br}{I_{\scriptscriptstyle{\text{\emph{KL}}}}}\right)dr\right)\\
\leq &\ \frac{b}{I_{\scriptscriptstyle \text{\emph {KL}}}}\left(1 + \int_{1}^\infty \mathbb P\left(\ell_{\kappa+br/I_{\scriptscriptstyle{\text{\emph{KL}}}},\kappa}\leq b\right)dr\right)\\
=&\ \frac{b}{I_{\scriptscriptstyle \text{\emph {KL}}}}\left(1 + \int_{1}^\infty \mathbb P\left(\frac{I_{\scriptscriptstyle{\text{\emph{KL}}}}}{br}\ell_{\kappa+br/I_{\scriptscriptstyle{\text{\emph{KL}}}},\kappa}\leq I_{\scriptscriptstyle{\text{\emph{KL}}}}(1-(r-1)/r)\right)dr\right)
\end{align*}
The concentration bound for $(1/t)\int_0^t f\circ\theta_sds$ was derived by \cite{reynaud2007some}, where $f$ is a bounded function on event data during $[-a,0)$ for some $a>0$ and $\theta_s$ translates the event time by $s$. They also provide a bound on the number of events in every unit time, which is followed immediately by a bound on the intensity. To be specific, we have $\lambda_0(s),\lambda_\infty(s)\leq O(\log t)$ with probability $1-O(t^{-c})$ for every $s\in[0,t]$, where $c$ is some large-enough constant. For $\kappa = 0$, the log-likelihood ratio can be written as
\begin{align*}
t^{-1}\ell_{t,0} =&\ t^{-1}\int_0^t (\lambda_{\infty}(s)-\lambda_0(s)+\lambda_0(s)\log(\lambda_0(s)/\lambda_\infty(s)))ds \\
&\ + t^{-1} \int_0^t\log(\lambda_0(s)/\lambda_\infty(s))(dN_s-\lambda_0(s)ds).
\end{align*}
For the first term, a concentration bound around $I_{\scriptscriptstyle{\text{\emph{KL}}}}$ can be derived using \cite{reynaud2007some}'s argument, by first applying the bound on $\lambda_\infty$ and $\lambda_0$ above. We have for any $0<u\leq O(\sqrt{t})$,
\begin{align*}
\mathbb P\left(\left|t^{-1}\int_0^t (\lambda_{\infty}(s)-\lambda_0(s)+\lambda_0(s)\log(\lambda_0(s)/\lambda_\infty(s)))ds - I_{\scriptscriptstyle{\text{\emph{KL}}}}\right|\geq c_1\sqrt{\frac{(u+\log t)u}{t}}\log^2t\right)\\
 \leq O(e^{-u}) + O(t^{-c}),
\end{align*}
where $c_1$ depends on the model and $c$. The second term is the average of a martingale, and will have a concentration bound around 0 by first applying the bound on $\lambda_\infty,\lambda_0$ and the number of events every unit time, followed by Hoeffding's inequality. We have for any $u>0,$
$$
\mathbb P\left(\left|t^{-1} \int_0^t\log(\lambda_0(s)/\lambda_\infty(s))(dN_s-\lambda_0(s)ds)\right|\geq c_2u\log^2 t\right)\leq O( e^{-2tu^2}) + O(t^{-c}),
$$
where $c_2$ depends on the model and $c$. With the two concentration inequality above, we can check that 
$
\mathbb E_0\left[T_{\scriptscriptstyle \text{\emph{C}}}|\mathcal H_0\right]
$
is $b(1+o(1))/I_{\scriptscriptstyle{\text{\emph{KL}}}}$. For $\kappa >0$, since the process restarts at $\kappa$, it can be regarded as a translation of the case $\kappa = 0$. The only difference is that $\lambda_\infty$ takes into consideration the events before $\kappa$. Since $\varphi$ has finite support, such difference only exists for a limited time and will not make a large difference for EDD analysis.

\end{document}